\documentclass{article}

\usepackage[preprint]{neurips_2026}

\usepackage[utf8]{inputenc} % allow utf-8 input
\usepackage[T1]{fontenc}    % use 8-bit T1 fonts
\usepackage{xurl}           % permit long bibliography URLs to wrap
\usepackage{amsmath,amssymb,amsfonts,amsthm}
\usepackage{booktabs}       % professional-quality tables
\usepackage{algorithm}
\usepackage{algorithmic}
\usepackage{graphicx}
\usepackage{adjustbox}
\usepackage{float}          % enables [H] to force figure placement at source
\usepackage{nicefrac}       % compact symbols for 1/2, etc.
\usepackage{microtype}      % microtypography
\usepackage{xcolor}         % colors
\usepackage{enumitem}       % list customization (itemsep, label= options)
\usepackage[hidelinks]{hyperref}

\newtheorem{theorem}{Theorem}
\newtheorem{corollary}{Corollary}

\theoremstyle{definition}

\theoremstyle{remark}

\theoremstyle{plain}

\newcommand{\thetahat}{\hat{\theta}}
\newcommand{\Var}{\mathrm{Var}}
\newcommand{\Cov}{\mathrm{Cov}}
\newcommand{\E}{\mathbb{E}}

\newcommand{\cF}{\mathcal{F}}
\newcommand{\dlim}{\xrightarrow{d}}
\newcommand{\plim}{\xrightarrow{p}}
\newcommand{\1}{\mathbbm{1}}
\usepackage{bbm}

\title{Active Multiple-Prediction-Powered Inference}
\author{
  Nicholas Peter Brawand \\
  The MITRE Corporation \\
  \texttt{nbrawand@mitre.org}
  \And
  Nima Leclerc \\
  The MITRE Corporation \\
  \texttt{nleclerc@mitre.org}
  \And
  Ant Ngo \\
  The MITRE Corporation \\
  \texttt{ango@mitre.org}
  \And
  Matthew Peterson \\
  The MITRE Corporation \\
  \texttt{mpeterson@mitre.org}
  \And
  Sriram Vishwanath \\
  The MITRE Corporation \\
  Georgia Institute of Technology \\
  \texttt{sriram@ece.gatech.edu}
  \And
  Laith Alhussein \\
  The MITRE Corporation \\
  \texttt{lalhussein@mitre.org}
  \And
  Ben Wellner \\
  The MITRE Corporation \\
  \texttt{wellner@mitre.org}
}

\begin{document}
\maketitle

% ============================================================
\begin{abstract}
% ============================================================
The growing use of artificial intelligence (AI) in science and industry is increasing the need for human labels and evaluation. Prediction-powered inference (PPI) can reduce labeling requirements by incorporating model predictions, but predictors differ in both accuracy and query cost, creating a joint challenge of allocating labeling resources. Standard formulations of PPI and active statistical inference (ASI) focus on a single predictor. We propose Active Multiple-Prediction-Powered Inference (AM-PPI), which routes each instance to a cost-appropriate predictor subset, samples gold-standard labels in proportion to the residual uncertainty of the subset, and reweights predictions to minimize estimator variance, all under a single deployment-time budget. We establish unbiasedness and asymptotic normality, yielding asymptotically valid confidence intervals (CIs). Experiments on synthetic regression, clinical consistency monitoring, adult income, and bank marketing show that AM-PPI can produce narrower CIs than single-predictor baselines with gains varying across predictor costs, error patterns, and labeling budgets.
\end{abstract}

% ============================================================
\section{Introduction}
\label{sec:intro}
% ============================================================

As AI systems become more capable and more deeply integrated into our economy, many tasks will be automated, and the value of human labor will continue to shift toward directing and verifying system outputs. Already, the volume of output from these systems far exceeds the capacity of humanity to verify all that is produced. This growing need calls for methods that increase the efficiency of our efforts to steer, measure, and verify AI systems. In scientific studies, industrial applications, and AI evaluations, measuring and verifying often requires reliable population-level inference despite limited access to gold-standard labels. In the case of AI evaluation, the parameter of interest can be system performance estimated from a sample of system outputs labeled by a human or a large language model (LLM). For example, \citet{keyes2026monitoringdeployedaisystems} describe a practical AI monitoring framework deployed at Stanford Health Care. In this setting, distribution shifts can degrade the performance of deployed clinical models~\citep{finlayson2021clinician, park2021reliable}, while limited access to labels poses a key challenge for statistically valid monitoring~\citep{dolin2025postdeployment}.

The common inference challenge is that collecting gold-standard labels can be costly and slow. Classical inference from a small labeled sample yields wide CIs. At the same time, AI predictions are abundant but potentially biased. Prediction-powered inference (PPI)~\citep{angelopoulos2023ppi} bridges this gap by combining a small labeled dataset with a large set of AI predictions to produce valid CIs that are tighter than classical inference alone.  Extensions of PPI are covered in Section~\ref{sec:related}. These methods sample uniformly overall or within predetermined strata.

Active statistical inference (ASI)~\citep{zrnic2024active} introduced a complementary idea: rather than sampling labels uniformly, allocate the labeling budget to instances where the predictor is most uncertain, thereby reducing the width of CIs while making the best use of human labeling efforts in an online setting. Robust ASI~\citep{li2025robust} improved the efficiency of the active sampling algorithm when uncertainty estimates are poorly calibrated, but both methods are restricted to a single predictor and rely on model uncertainty to function. Prediction pipelines may offer several models with different query costs and strengths, such as a compact model, a larger model, and models trained on complementary feature panels. Furthermore, these models may be poorly calibrated, making their uncertainty estimates unreliable. For example, clinical consistency monitoring can use LLMs of different types, including encoders and decoders with varying strengths and costs. Neither PPI nor ASI can exploit this structure.

We propose Active Multiple-Prediction-Powered Inference (AM-PPI), which jointly optimizes three coupled decisions under a unified budget constraint that accounts for predictor quality, query costs, and labeling costs. AM-PPI determines (1) \emph{which subset of predictors} to query for each instance; (2) \emph{whether to collect a gold-standard label}, sampling proportional to the estimated residual of the chosen predictor subset; and (3) \emph{how to weight the predictions} from each predictor in the subset to minimize estimator variance. Our contributions are (i) conditional optimality conditions for these three decisions together with practical calibration and deployment algorithms for budgeted mean estimation (Section~\ref{sec:method}), recovering the ASI estimator as a special case; (ii) asymptotic normality of the AM-PPI estimator yielding valid CIs (Section~\ref{sec:theory}); and (iii) empirical validation on synthetic regression, clinical consistency monitoring, income-rate estimation, and bank subscription-rate estimation, measuring CI width and empirical coverage relative to other baselines across costs and budgets and demonstration of the dynamic routing behavior of AM-PPI to minimize interval widths under specific single-predictor cost conditions over total budget sweeps (Section~\ref{sec:experiments}).

% ============================================================
\section{Related work}
\label{sec:related}
% ============================================================

\paragraph{Prediction-powered inference.} PPI~\citep{angelopoulos2023ppi} introduced the idea of combining a small labeled dataset with a large ML-predicted dataset to form valid CIs for population parameters. PPI++~\citep{angelopoulos2023ppiplus} added a tuning parameter $\lambda$ that interpolates between classical and prediction-powered estimators, ensuring PPI does not increase variance. Stratified PPI~\citep{fisch2024stratppi} further tightens intervals by stratifying on prediction quality. Cross-PPI~\citep{zrnic2024crossppi} uses cross-fitting to avoid overfitting the ML model to labeled data. Stratified PPI~\citep{fisch2024stratppi} partitions the data into strata based on autorater predictions and applies stratum-specific tuning, substantially reducing variance when prediction quality varies across subdomains. PSPS~\citep{miao2024psps} provides a task-agnostic framework for post-prediction inference. Multiple-PPI~\citep{cowenbreen2026multippi} extends PPI to multiple predictors, solving for a cost-aware global allocation of query counts across predictors under uniform labeling for prospective experimental design. AM-PPI differs in three important ways: (i) we adaptively route individual instances to predictor subsets based on their estimated residual variance rather than fixing a global per-predictor query count; (ii) we jointly optimize active label sampling alongside predictor selection; and (iii) our setting is inference on a given pool of unlabeled instances, including monitoring batches, rather than planning a future experiment.

\paragraph{Active statistical inference.} ASI~\citep{zrnic2024active} designs an optimal sampling rule $\pi(x)$ that prioritizes labeling instances where the predictor is uncertain, yielding substantially tighter CIs than uniform sampling. Robust ASI~\citep{li2025robust} improves active sampling when uncertainty estimates are inaccurate by interpolating between active and uniform sampling along a budget-preserving path. Both methods use a single predictor, and AM-PPI extends ASI from one predictor to multiple predictors with cost-aware subset routing. AM-PPI recovers the ASI estimator when a single predictor has unit weight. The learned residual-based policy of AM-PPI and the predictor-uncertainty policy of ASI can differ. We focus on the comparison with ASI but future work will include robust optimization extensions.

\paragraph{Post-deployment monitoring.} One of our examples focuses on evaluating AI models in a healthcare setting. \citet{dolin2025postdeployment} frame post-deployment monitoring in healthcare as a collection of hypothesis tests and identify label scarcity as the central open challenge. Label-free approaches to performance estimation under distribution shift exist~\citep{chen2022estimating, garg2022leveraging}, but lack the statistical guarantees required by regulatory frameworks. Sequential shift detection methods~\citep{amoukou2024sequential} can flag potential changes in performance, but cannot quantify the magnitude of degradation without labels. AM-PPI can support the estimation component of such monitoring by constructing CIs for a population mean or rate under a shared budget for prediction and labeling, as demonstrated in Section~\ref{sec:experiments}.

% ============================================================
\section{Active multiple-prediction-powered inference}
\label{sec:method}
% ============================================================

\subsection{The AM-PPI estimator}

We observe $n$ unlabeled instances $X_1, \ldots, X_n$ drawn i.i.d.\ from $P_X$. Each instance has an unobserved gold-standard label $Y_i$ that is expensive to obtain (e.g., expert review). We have access to $k$ predictors $f_1, \ldots, f_k$ with per-predictor query costs $c_1, \ldots, c_k$; a subset $I \subseteq \{1, \ldots, k\}$ of predictors has combined query cost $c_I = \sum_{j \in I} c_j$. The total budget $B$ covers both predictor queries and label collection (at cost $c_{\text{label}}$ per label). Our goal is to estimate $\theta^* = \E[Y]$ and construct a valid, minimal-length $(1-\alpha)$ CI subject to the budget constraint; we develop mean estimation here, which does not directly cover pairwise metrics such as AUROC. AM-PPI extends PPI++~\citep{angelopoulos2023ppiplus} and ASI~\citep{zrnic2024active} by introducing a per-instance predictor query subset $I_i \subseteq \{1, \ldots, k\}$ with predictions $f_{I_i}(X_i) \in \mathbb{R}^{|I_i|}$ and a weight vector $\lambda_{I_i} \in \mathbb{R}^{|I_i|}$ that combines them. Let $\xi_i \sim \mathrm{Bern}(\pi_{I_i}(X_i))$ indicate whether $Y_i$ is collected, with subset-dependent sampling probability $\pi_{I_i}(X_i) \in (0,1]$. The AM-PPI estimator is:
\begin{equation}
    \hat{\theta} = \frac{1}{n} \sum_{i=1}^{n} \left[ \lambda_{I_i}^\top f_{I_i}(X_i) + \left( Y_i - \lambda_{I_i}^\top f_{I_i}(X_i) \right) \frac{\xi_i}{\pi_{I_i}(X_i)} \right].
    \label{eq:amppi}
\end{equation}
When $k=1$ and $\lambda = 1$, this reduces to the ASI estimator of \citet{zrnic2024active}.

\subsection{Variance minimization under budget constraints}

The asymptotic variance of the AM-PPI estimator is (see Appendix~\ref{app:proof_clt}):
\begin{equation}
    \begin{aligned}
    \Var(\hat{\theta}) &= \frac{1}{n}\,\Var(Y) \\
    &\quad + \frac{1}{n}\,\E\!\left[ r_{I(X)}(X)\!\left(\frac{1}{\pi_{I(X)}(X)} - 1\right) \right],
    \end{aligned}
    \label{eq:variance}
\end{equation}
where $r_I(x) := \E\bigl[(Y - \lambda_I^\top f_I(x))^2 \mid X = x\bigr]$
is the conditional residual second moment under predictor subset~$I$ and
combination weights~$\lambda_I$. We let $u_I(x) := \sqrt{r_I(x)}$ denote the per-instance root mean squared residual, our population-level measure of model uncertainty. Our goal is to construct the shortest valid $(1-\alpha)$ CI for $\theta^*$ subject to a per-instance budget constraint. Since the interval width scales with $\sqrt{\Var(\hat{\theta})}$, this reduces to minimizing~\eqref{eq:variance} jointly over the labeling policy~$\pi$, the combination weights~$\lambda$, and the subset policy~$I$. The $\Var(Y)$ term does not depend on $(\pi,\lambda,I)$ and therefore drops out of the optimization:
\begin{equation}
    \begin{aligned}
    \min_{\pi,\, \lambda,\, I} \;&
    \E\!\left[ r_{I(X)}(X)\!\left(\frac{1}{\pi_{I(X)}(X)} - 1\right) \right] \\
    \text{s.t.} \;&
    \E\!\left[ c_{I(X)} + \pi_{I(X)}(X)\,c_{\mathrm{label}} \right] \leq b,
    \end{aligned}
    \label{eq:opt}
\end{equation}
where $b = B/n$ is the per-instance budget. Introducing a Lagrange
multiplier $\mu \geq 0$ for the budget constraint gives:
\begin{equation}
    \begin{aligned}
    \mathcal{L}
    &= \frac{1}{n}\,\E\!\left[ r_{I(X)}(X)\!\left(\frac{1}{\pi_{I(X)}(X)} - 1\right) \right] \\
    &\quad + \mu \!\left(
        \E\!\left[ c_{I(X)} + \pi_{I(X)}(X) \cdot c_{\mathrm{label}} \right] - b
    \right).
    \end{aligned}
\end{equation}

Taking the functional derivative with respect to $\pi(x)$ and setting it to zero yields the unconstrained optimum (Appendix~\ref{app:optimization})
\begin{equation}
    \pi^*_I(x) = \frac{u_I(x)}{\sqrt{n\,\mu \, c_{\text{label}}}}.
    \label{eq:pi_star}
\end{equation}
The sampling rule is proportional to the per-instance uncertainty $u_I(x)$, so instances with higher model uncertainty are sampled more often. Since $\pi$ must be a valid probability, AM-PPI uses the clipped quantity
\begin{equation}
    \hat\pi_I(x) := \min\!\bigl(1,\, \pi^*_I(x)\bigr)
    \label{eq:pi_clipped}
\end{equation}
in all downstream substitutions. Differentiating with respect to $\lambda_I$ yields the weighted least-squares condition
\begin{equation}
    \begin{gathered}
    \E\!\left[w_I(X)\,(Y - \lambda_I^\top f_I(X))\,f_I(X)\right] = 0, \\
    w_I(x) = \1\{I(x)=I\}\!\left(\frac{1}{\hat\pi_I(x)} - 1\right),
    \end{gathered}
    \label{eq:lambda_star}
\end{equation}
where the indicator restricts the condition to instances routed to subset $I$, and $w_I=0$ for instances always labeled. For an active budget constraint attained continuously, the multiplier $\mu$ satisfies
\begin{equation}
    \E\!\left[c_{I(X)} + c_{\text{label}}\,\hat\pi_{I(X)}(X)\right] = b,
    \label{eq:mu_condition}
\end{equation}
With the residual models fixed, the expected cost of the selected policy is non-increasing in $\mu$, and we find a feasible multiplier by bisection. Routing switches can introduce jumps, so an exact budget match or a unique multiplier need not exist. In the unclipped, single-fixed-subset case, Eq.~\eqref{eq:mu_condition} admits the closed form $\mu = (c_{\text{label}}/n)\bigl(\E_X[\sqrt{r_I(X)}]/(b - \E[c_I])\bigr)^2$. In general we recompute routing at each candidate $\mu$. For fixed weights and multiplier, the subset minimizes the per-instance Lagrangian cost
\begin{equation}
    \begin{aligned}
    I^*(x) = \arg\min_{I} \!\biggl\{& \frac{r_I(x)}{n}\!\left(\frac{1}{\hat\pi_I(x)} - 1\right) \\
    &+ \mu\,c_I + \mu\,c_{\text{label}}\,\hat\pi_I(x) \biggr\},
    \end{aligned}
    \label{eq:I_star}
\end{equation}
which on the unclipped set collapses to $\frac{2\sqrt{\mu c_{\text{label}}}}{\sqrt{n}}\sqrt{r_I(x)} - r_I(x)/n + \mu c_I$ (Appendix~\ref{app:optimization}) and is non-decreasing in $r_I(x)$ across both regimes, capturing the per-instance tradeoff between query cost and residual variance.

\subsection{Practical algorithms}

The conditional optimality conditions above depend on the unknown $r_I(x)$, which we estimate from a burn-in dataset of $N$ fully-labeled samples and refine iteratively (Algorithm~\ref{alg:ampi_calib}). Algorithm~\ref{alg:ampi_deploy} then deploys the calibrated quantities $(\hat{\lambda}_I, \hat{u}_I, \hat{\mu})$ to a stream of $n$ unlabeled instances, producing the estimator $\hat\theta$ and CI $\mathcal{C}_\alpha$.

\begin{algorithm}[tbp]
\caption{AM-PPI Calibration}
\label{alg:ampi_calib}
\begin{algorithmic}[1]
\REQUIRE Burn-in data $\{(X_j, Y_j, f_1(X_j), \ldots, f_k(X_j))\}_{j=1}^{N}$; candidate subsets $\mathcal{I} \subseteq 2^{\{1,\ldots,k\}}$; costs $\{c_I\}_{I \in \mathcal{I}}, c_{\text{label}}$; per-instance budget $b$; deployment sample size $n$; max iterations $T$
\FOR{each $I \in \mathcal{I}$}
    \STATE Initialize $\hat{\lambda}_I \leftarrow \widehat{\Cov}(f_I, f_I)^{-1}\, \widehat{\Cov}(f_I, Y)$.
    \STATE Fit $\hat{u}_I(\cdot)^2$ by regressing $(Y - \hat{\lambda}_I^\top f_I(X))^2$ on $X$.
\ENDFOR
\FOR{$t = 1, \ldots, T$}
    \STATE $\hat{r}_I(x) \leftarrow \hat{u}_I(x)^2$
    \STATE Find $\hat{\mu}$ by bisection in Eq.~\eqref{eq:mu_condition}, evaluating $I(x)$ at each candidate $\hat{\mu}$.
    \STATE $\hat{\lambda}_I \leftarrow$ weighted least squares on all burn-in cases with weights $1/\hat\pi_I-1$
    \STATE $\hat{u}_I(\cdot)^2 \leftarrow$ refit $(Y - \hat{\lambda}_I^\top f_I(X))^2$
    \STATE \textbf{break} if the relative parameter drift is below tolerance or stops improving
\ENDFOR
\STATE Re-solve $\hat\mu$ with the final uncertainty models.
\STATE \textbf{return} $\hat{\lambda}_I,\; \hat{u}_I,\; \hat{\mu}$
\end{algorithmic}
\end{algorithm}

\begin{algorithm}[tbp]
\caption{AM-PPI Deployment}
\label{alg:ampi_deploy}
\begin{algorithmic}[1]
\REQUIRE Unlabeled instances $X_1, \ldots, X_n$; calibrated $\hat{\lambda}_I, \hat{u}, \hat{\mu}$ from Algorithm~\ref{alg:ampi_calib}; candidate subsets $\mathcal{I}$; costs $\{c_I\}_{I \in \mathcal{I}}, c_{\text{label}}$; confidence level $\alpha$
\FOR{$i = 1, \ldots, n$}
    \STATE Compute $\hat\pi_I(X_i)$ for each $I \in \mathcal{I}$ using Eq.~\eqref{eq:pi_clipped} with calibrated $\hat u_I$ and $\hat\mu$.
    \STATE Select $I_i$ using Eq.~\eqref{eq:I_star} with calibrated $\hat u_I$ and $\hat\mu$; query $f_{I_i}(X_i)$.
    \STATE $\hat{\pi}(X_i) \leftarrow \hat\pi_{I_i}(X_i)$;\; $\xi_i \sim \mathrm{Bern}(\hat{\pi}(X_i))$;\; if $\xi_i = 1$, collect $Y_i$
    \STATE Set $Z_i \leftarrow \hat{\lambda}_{I_i}^\top f_{I_i}(X_i)
    + (Y_i - \hat{\lambda}_{I_i}^\top f_{I_i}(X_i))\,\xi_i/\hat\pi(X_i)$.
\ENDFOR
\STATE $\hat{\theta} \leftarrow n^{-1}\sum_{i=1}^n Z_i$;
    $\hat{\sigma}^2 \leftarrow n^{-1}\sum_{i=1}^n (Z_i-\hat\theta)^2$.
\STATE \textbf{return} $\hat{\theta},\; \mathcal{C}_\alpha = \hat{\theta} \pm z_{1-\alpha/2} \cdot \hat{\sigma} / \sqrt{n}$
\end{algorithmic}
\end{algorithm}

The burn-in labels are a separate calibration resource and are excluded from the deployment budget $B=nb$. The practical weight update uses all calibration cases for each subset, omitting the routing indicator in Eq.~\eqref{eq:lambda_star}; it is therefore a surrogate update rather than the exact routed weight minimizer. In the numerical audit, $90/100$ configurations meet the stated two-sweep update tolerance (Appendix~\ref{sec:rebuttal-diagnostics}). The deployment uncertainty models $\hat{u}_I$ are trained on covariates $X$ alone. In practice the uncertainty models need not operate on the same feature representation as the predictors themselves and can instead rely on lightweight covariates that are cheap to extract (e.g., the length of a discharge summary rather than its full text), as we demonstrate in Section~\ref{sec:experiments}.

% ============================================================
\section{Theoretical analysis}
\label{sec:theory}
% ============================================================

\subsection{Asymptotic normality and coverage}
\label{sec:asymptotic}

We establish that the AM-PPI estimator is asymptotically normal and produces valid CIs. Throughout this section and Appendix~\ref{app:proof_clt}, starred quantities denote limits of the calibration estimates, which need not minimize the joint objective. The burn-in sample is independent of deployment observations.

\begin{theorem}[Asymptotic normality of AM-PPI]
\label{thm:clt}
Suppose the following regularity conditions hold: (i) the moments $\E[Y^4]$ and $\E[\|f_I(X)\|^4]$ are bounded for all $I \in \mathcal{I}$; (ii) the estimated quantities $\hat{\pi}$, $\hat{\lambda}_I$, and $\hat{I}$ converge in probability to their population counterparts $\pi^*$, $\lambda_I^*$, and $I^*$; (iii) the sampling probabilities are bounded away from zero, $\pi(x) \geq \pi_{\min} > 0$ for all $x$. Then the AM-PPI estimator satisfies
\[
\sqrt{n}(\hat{\theta} - \theta^*) \xrightarrow{d} \mathcal{N}(0, V),
\]
where $V = \Var(Y) + \E\!\left[(Y - \lambda_{I^*}^\top f_{I^*}(X))^2\!\left(\tfrac{1}{\pi^*(X)} - 1\right)\right]$ is the asymptotic variance from Eq.~\eqref{eq:variance}.
\end{theorem}
See Appendix~\ref{app:proof_clt} for full proof.

\begin{corollary}[Valid CIs]
\label{cor:ci}
Under the conditions of Theorem~\ref{thm:clt}, with $V>0$, the variance estimator $\hat{\sigma}^2$ is consistent for $V$ (Appendix~\ref{app:proof_ci}), and the CI $\mathcal{C}_\alpha = \hat{\theta} \pm z_{1-\alpha/2}\, \hat{\sigma}/\sqrt{n}$ satisfies
\[
\lim_{n \to \infty} \mathbb{P}(\theta^* \in \mathcal{C}_\alpha) \geq 1 - \alpha.
\]
\end{corollary}
This follows from Theorem~\ref{thm:clt} and Slutsky's theorem; see Appendix~\ref{app:proof_ci}.

% ============================================================
\section{Experiments}
\label{sec:experiments}
% ============================================================

We evaluate AM-PPI on synthetic regression, MIMIC clinical consistency monitoring, adult income, and bank marketing. We first compare the methods across datasets, predictor costs, and budgets to measure the average relative performance of AM-PPI against other baselines. We then examine AM-PPI in more detail with selected predictor costs that demonstrate the different possible AM-PPI routing strategies and single-predictor crossover regimes that can arise with different total budgets.

\subsection{Experimental setup}

For each dataset, the methods target a population mean $\theta^* = \E[Y]$ and report 90\% CIs ($\alpha = 0.1$). Each evaluation averages results over repeated trials. The total budget $B$ covers only the deployment-time predictor queries and label collection. Gold-label cost is fixed at $c_{\mathrm{label}}=1$. Each task uses two predictors with different features or model capacities. Uncertainty models use covariates available before predictor queries.

We compare against PPI, PPI++, Stratified PPI, ASI, single-predictor AM-PPI, and a sequential baseline we introduce called Cascade. Cascade uses the prediction weights and residual uncertainty models of AM-PPI. It queries the cheap predictor for every instance and additionally queries the expensive predictor for a calibration-tuned fraction of instances with the highest estimated residual uncertainty. We evaluate each single-predictor baseline with both predictors separately, denoting the variants Method (cheap) and Method (expensive) according to their query costs. ASI baselines implement the original batch-mode active statistical inference method of \citet{zrnic2024active}, using unit weight $\lambda=1$ and uncertainty-proportional label sampling. The sampling rule and calibration details are given in Appendix~\ref{app:exp_details}. The $\lambda$-tuned single-predictor restriction of AM-PPI (AM-PPI-1), PPI, PPI++, Cascade, and other baselines are described in Appendix~\ref{app:comparisons}.

\paragraph{Synthetic regression.}
We generate $n = 20{,}000$ instances with $d = 5$ features drawn from a standard normal distribution. Instances with $X_1>-0.5$ comprise approximately 70\% of the data and follow the linear relationship $Y=2X_2+\varepsilon$ with low noise $\varepsilon\sim\mathcal{N}(0,0.1^2)$. The remaining instances follow $Y=\sin(3X_2X_3)+\varepsilon$ with higher noise $\varepsilon\sim\mathcal{N}(0,2^2)$. The cheap predictor is a decision tree of depth 2, and the expensive predictor is a decision tree of depth 6.

\paragraph{AI-generated MIMIC-III electronic healthcare record discharge summary consistency monitoring.}
We construct a semi-synthetic monitoring benchmark in which an
AI-powered tool generates discharge summaries from structured electronic healthcare record (EHR) data,
and we wish to estimate the rate at which these summaries are
consistent with the underlying structured data~\citep{chung2025verifact}. Using the MIMIC-III v1.4 database~\citep{johnson2016mimic}, we extracted and processed laboratory measurements recorded on the date of admission (``admission labs'') from the structured EHR to generate corresponding ``Labs on Admission'' summary sections via few-shot prompting of GPT-OSS-120B~\citep{openai2025gptoss120bgptoss20bmodel}. To construct
benchmark labels, we pair each summary with
both the original admission labs record (label $1$) and a perturbed copy
(label $0$), where perturbations are applied to a random subset of
lab entries within a record through value mutations (10--20\% deviation), label swaps,
and abnormal-flag inversions, following the error taxonomy of
EHRCon~\citep{kwon2024ehrcon}. We generate two difficulty levels by perturbing 15\% of the lab entries within each admission record for cases that are hard to detect and 50\% for cases that are easy to detect. The consistency classifiers serving as cheap and expensive predictors are
Nemotron-3-Nano-30B and Nemotron-3-Super-120B~\citep{nvidia2025nvidianemotron3efficient} respectively, both
queried with few-shot prompting at temperature 0.1. The full prompt is given in Appendix~\ref{app:mimic_prompts}. Uncertainty models use 11 easy-to-derive features, including note length and the number of lab measurements,
rather than the full clinical text. The labels indicate whether the EHR was perturbed and do not certify clinical consistency through human review.

\paragraph{Adult income-rate estimation.} Using the University of California--Irvine (UCI) Adult dataset~\citep{becker1996adult}, we estimate the fraction of people with annual income above \$50,000 within a population. The cheap predictor is a depth-1 gradient boosting classifier with 100 trees using demographic and education features. The expensive predictor includes employment and financial features with 200 depth-3 trees. Both are trained at the observed class frequencies. Predictor and cost selection use only training and internal calibration folds as described in Appendix~\ref{app:exp_details}.

\paragraph{Bank marketing subscription-rate estimation.} We use the 41,188-record \texttt{bank-additional-full} dataset from UCI Bank Marketing~\citep{moro2014bank} to estimate the term-deposit subscription fraction in the historical campaign mixture. Call duration is excluded because it is unavailable before the call ends. The cheap classifier uses seven customer-profile attributes with 100 depth-2 trees, and the expensive classifier uses all 19 retained attributes with 100 depth-3 trees. This random-split retrospective experiment does not establish performance on future campaigns.

\subsection{Comparison across datasets, costs, and budgets}

We compare the methods across the four datasets. For each dataset, we randomly draw three pairs of predictor costs and four  budgets per pair. One duplicate budget is removed from the results, leaving 47 dataset–cost–budget combinations. For each combination, we run 25 Monte Carlo simulations, each using a new bootstrap sample with replacement of held-out test records and fresh label-sampling draws and evaluate all 15 methods under the same costs and budget, with each method independently selecting labels according to its own sampling procedure. This yields 1,175 paired trials and 17,625 method evaluations. The sampling procedure is described in Appendix~\ref{sec:rebuttal-comparison}.

Table~\ref{tab:grand_overall} reports average interval-width ranks, improvements relative to classical estimation, and empirical coverage. AM-PPI with and without an intercept attain the two best average width ranks among 15 methods, with mean coverage $91.3\%$ and $91.0\%$, respectively. Their average interval-width improvements relative to classical estimation are $20.5\%$ and $20.2\%$, respectively. AM-PPI without an intercept is $6.1\%$ narrower than Cascade in the geometric mean of paired-trial width ratios. Per-dataset ranks and paired Cascade comparisons are reported in Appendix~\ref{sec:rebuttal-comparison}.

\begin{table*}[tp]
\centering
\small
\caption{Grand comparison across 47 dataset, cost, and budget scenarios and 1,175 paired trials, with a 90\% coverage target. Average percent improvement is the mean of $100(w_{\mathrm{classical}}-w)/w_{\mathrm{classical}}$, where $w$ is the confidence interval width of the method and $w_{\mathrm{classical}}$ is the classical confidence interval width. Classical uses uniform Bernoulli labels at the same total budget with no predictor-query costs. Rank 1 is narrowest (ties within 2.5\%); classical is excluded from the 15-method ranking. Coverage SD is the standard deviation of coverage across scenarios.}
\label{tab:grand_overall}
\begin{adjustbox}{max width=\textwidth}
\begin{tabular}{lrrrr}
\toprule
Method & Avg. width rank & \shortstack{Average percent\\improvement} & Avg. coverage & Coverage SD \\
\midrule
AM-PPI (intercept, ours) & 3.24 & +20.5 & 0.913 & 0.054 \\
AM-PPI (ours) & 3.40 & +20.2 & 0.910 & 0.056 \\
AM-PPI-1 (expensive, ours) & 4.77 & +14.9 & 0.906 & 0.048 \\
Cascade (ours) & 5.19 & +14.7 & 0.913 & 0.063 \\
ASI (expensive) & 5.63 & +12.8 & 0.908 & 0.061 \\
AM-PPI-1 (expensive, intercept, ours) & 5.71 & +11.4 & 0.895 & 0.051 \\
AM-PPI-1 (cheap, intercept, ours) & 6.34 & +12.4 & 0.901 & 0.059 \\
AM-PPI-1 (cheap, ours) & 7.34 & +9.2 & 0.899 & 0.073 \\
Stratified PPI (expensive) & 7.56 & +10.8 & 0.906 & 0.060 \\
PPI++ (expensive) & 8.22 & +8.0 & 0.897 & 0.056 \\
PPI (expensive) & 9.38 & +4.6 & 0.911 & 0.051 \\
ASI (cheap) & 9.49 & +4.3 & 0.891 & 0.062 \\
Stratified PPI (cheap) & 10.55 & +4.7 & 0.925 & 0.052 \\
PPI++ (cheap) & 10.58 & +4.6 & 0.911 & 0.061 \\
PPI (cheap) & 12.17 & -5.3 & 0.889 & 0.076 \\
\bottomrule
\end{tabular}
\end{adjustbox}
\end{table*}

\subsection{Budget-precision tradeoffs at fixed predictor costs}

Table~\ref{tab:grand_overall} shows the strong average relative performance of AM-PPI across datasets, costs, and budgets for 1,175 paired trials. It does not show the interesting routing behaviors of AM-PPI and single predictor methods across a budget sweep, which is the purpose of this subsection. Toward this end, we examine how AM-PPI's routing can reduce interval width across budgets where the relative advantages of predictor strategies are not known a priori. Depending on predictor costs, and accuracy, one predictor may yield narrower intervals at low budgets and another at high budgets, with similar performance near the crossover. Alternatively, a mixed strategy may outperform both single-predictor strategies throughout the budget range, or the routing of AM-PPI may remain concentrated on one predictor throughout the budget range.

To examine these regimes, we sweep the total budget $B$ while holding predictor costs fixed and evaluate interval width (first column), coverage (middle column), and relative percent improvement (third column) using Monte Carlo simulations. Figure~\ref{fig:experiments} compares AM-PPI with ASI, the strongest prior published baseline on average, applied separately to each predictor. These curves use $2{,}000$ trials per budget for synthetic regression and $1{,}000$ for each empirical task. Each trial resamples test rows with replacement and draws labels independently, holding calibration fixed. Appendix~\ref{app:ablation} provides a hindsight mechanism illustration of the tradeoff between cost and accuracy, using realized outcomes to construct oracle uncertainties. Results for additional predictor costs are reported in Appendix~\ref{sec:rebuttal-costs}.

\paragraph{Synthetic regression.} The first row of Figure~\ref{fig:experiments} uses predictor costs $0.01$ and $0.087$ for the cheap and expensive predictors respectively. The additional cost of the more accurate model is substantial at lower budgets so that the expensive-predictor baselines cannot operate at the smallest budgets. Within the single predictor crossing region, a mixed routing strategy can trade prediction accuracy against query cost to reduce interval widths while AM-PPI dynamically switches to a single predictor strategy at the budget limits (see Table~\ref{tab:diagnostics_regression}). 

\paragraph{MIMIC-III EHR-discharge consistency monitoring.} This setup is similar to the synthetic regression with a clear crossing regime between single predictor strategies but is focused on a more realistic scenario of evaluating the quality of AI generated discharge summaries. The second row of Figure~\ref{fig:experiments} compares interval widths and coverage under predictor costs $0.1$ and $0.22$. Here too we see similar dynamic routing behavior by AM-PPI to reduce interval widths. 

\paragraph{Adult income-rate estimation.} This setup is different from the previous two examples in that a mixed strategy across the three predictor subsets (cheap, expensive, and both cheap+expensive) is advantageous across the budget sweep (see Table~\ref{tab:diagnostics_adult_income}). The third row of Figure~\ref{fig:experiments} uses costs $0.001$ and $0.005$ where the two single predictors nearly overlap across the 12 budgets instead of cross. AM-PPI intervals are $9.5\%$ narrower in geometric mean than the better ASI and $7.5\%$ narrower than the better singleton AM-PPI. 

\paragraph{Bank marketing subscription-rate estimation.} Lastly, we show an example where AM-PPI’s routing remains concentrated on the cheap predictor across the budget range. Here, AM-PPI gains are more modest. The fourth row of Figure~\ref{fig:experiments} uses query costs $0.01$ and $0.05$ across 12 budgets where AM-PPI routes 90\%--94\% of examples to the cheap predictor alone and queries the expensive predictor for the remainder (see Table~\ref{tab:diagnostics_bank_marketing}). The improvement against ASI is $7.1\%$ at lower budgets and the better singleton AM-PPI is only $1.7\%$, indicating that most of the advantage over unit-weight ASI is already available from a single predictor with learned weighting and the residual sampling of AM-PPI.

\begin{figure}[p]
    \centering
    \includegraphics[width=0.94\textwidth]{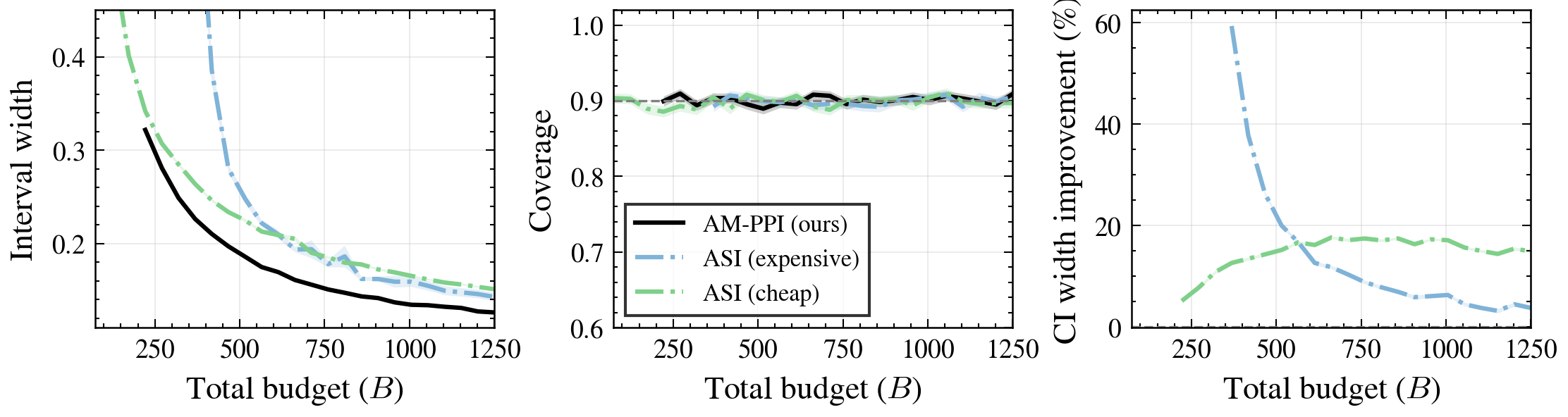}\\[3pt]
    \includegraphics[width=0.94\textwidth]{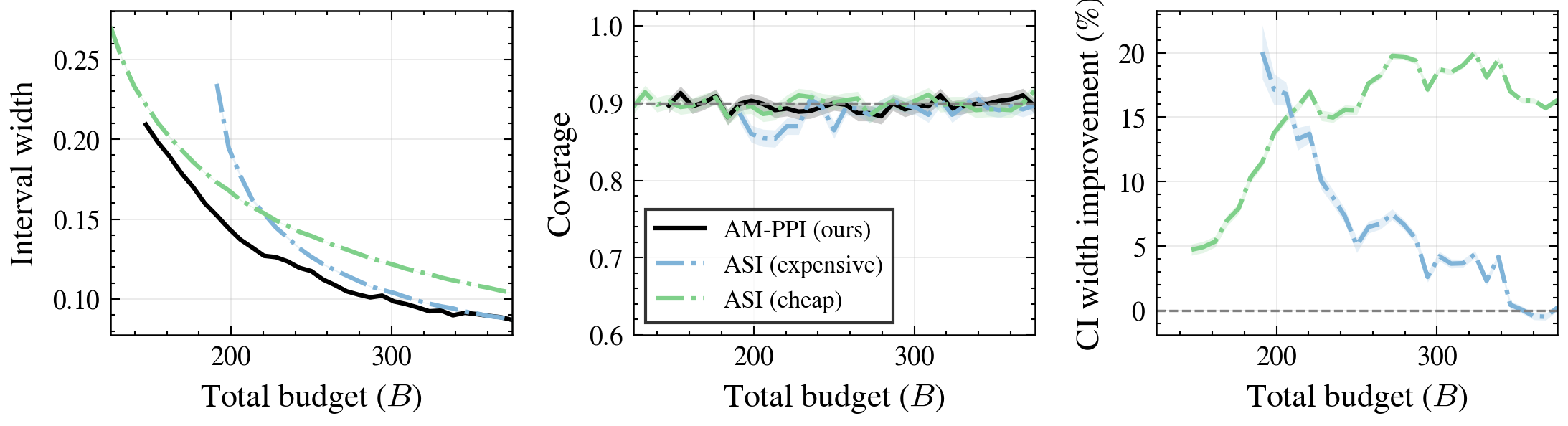}\\[3pt]
    \includegraphics[width=0.94\textwidth]{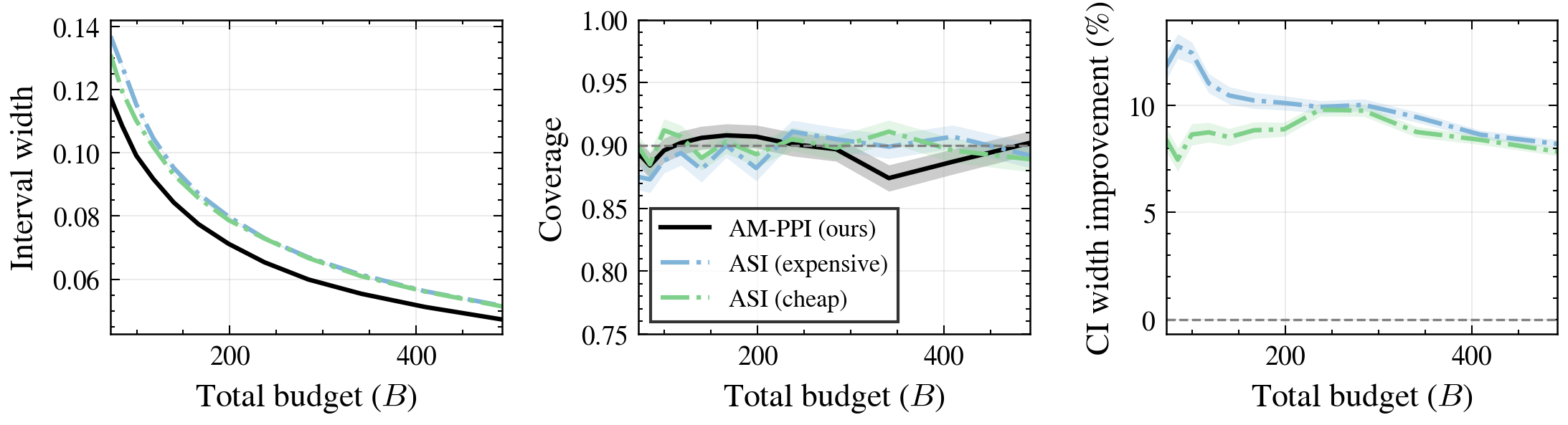}\\[3pt]
    \includegraphics[width=0.94\textwidth]{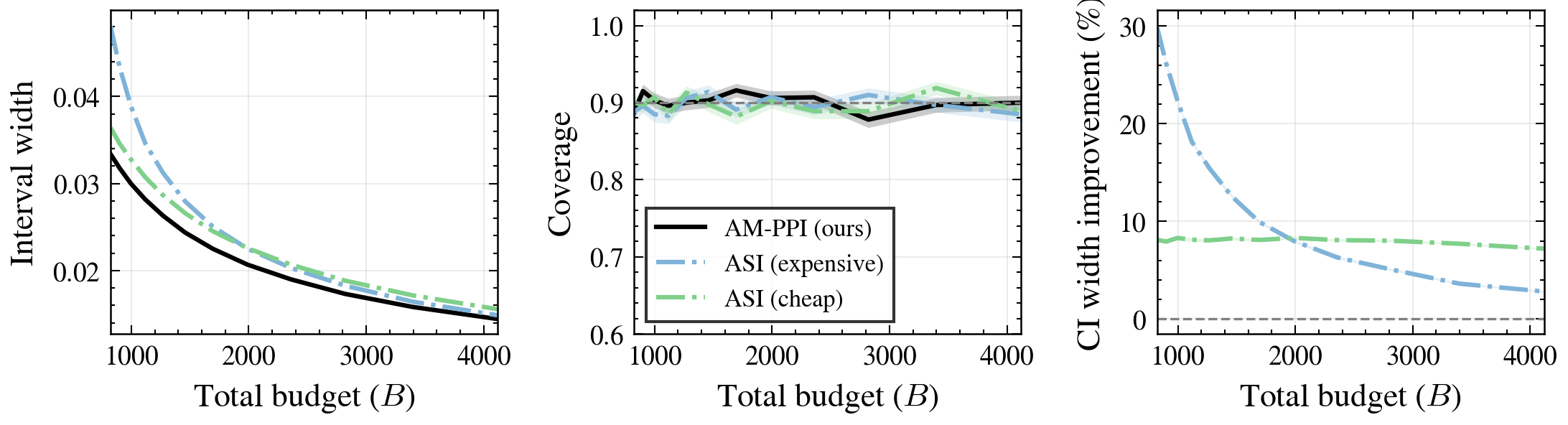}
    \caption{Interval width and empirical coverage across budgets at fixed predictor costs demonstrating different possible routing behaviors and single-predictor crossover regimes. Rows one through four show synthetic regression, MIMIC EHR-discharge consistency, Adult income-rate estimation, and Bank Marketing subscription-rate estimation. Left panels show CI width as a function of total budget. Middle panels show empirical coverage with the 90\% target indicated by a dashed line. Right panels show the mean paired-trial percent improvement $100(w_{\mathrm{ASI}}-w_{\mathrm{AM\text{-}PPI}})/w_{\mathrm{ASI}}$ for each ASI predictor, where $w$ denotes interval width. Positive values indicate narrower AM-PPI intervals. ASI curves use the original method of \citet{zrnic2024active} with uncertainty-proportional sampling. Additional baselines are described in Appendix~\ref{app:comparisons} and compared in Table~\ref{tab:grand_overall}. Shaded bands show $\pm 1$ standard error of the mean (SEM) across independent trials.}
    \label{fig:experiments}
\end{figure}

% ============================================================
\section{Discussion}
\label{sec:discussion}
% ============================================================

Across datasets, predictor costs, and budgets, AM-PPI produced narrower confidence intervals on average than competing baselines with expected coverage (Table~\ref{tab:grand_overall}). The budget sweeps at fixed predictor costs (Figure~\ref{fig:experiments}) and the hindsight synthetic ablation study (Appendix~\ref{app:ablation}) help illustrate two mechanisms that can contribute to these gains. First, even with a single predictor, AM-PPI learns the prediction weight $\lambda$ and allocates labels according to estimated residual uncertainty. This differs from the probability-based uncertainty used by ASI for binary outcomes. Further work is needed to isolate the contribution of residual-based sampling from that of the weights. Second, access to multiple predictors enables AM-PPI to route instances to different predictors or combinations of predictors, including the joint cheap-and-expensive subset. These choices allow AM-PPI to exploit complementary predictor strengths while balancing prediction quality against query costs.

\paragraph{Limitations.} \emph{Calibration.} AM-PPI requires a labeled burn-in dataset to calibrate residual error surrogates. Limited calibration data or suboptimal optimization can degrade both the labeling policy and subset routing, similar to how poor probability calibration can impair ASI. In such settings, restricting AM-PPI to a single predictor may yield narrower CIs (Table~\ref{tab:grand_overall}). Calibration also incurs an upfront labeling cost, which is excluded from the deployment budgets reported here. One practical adoption strategy is to use a baseline such as PPI for an initial CI while calibrating AM-PPI for subsequent evaluations and post-deployment monitoring. \emph{Scalability.} Considering all nonempty subsets of $k$ predictors requires $2^k-1$ candidates. Exhaustive enumeration is optional, however, and restricting the subsets can reduce computational cost. \emph{Drift.} Distribution shift can affect all methods considered here, including AM-PPI and Cascade. Appendix~\ref{sec:rebuttal-shift} examines calibration shift and constructed covariate shifts, but transfer to new institutions and future time periods remains untested. Detecting drift and determining when and how to recalibrate remain directions for future work.

% ============================================================
\section{Conclusion}
\label{sec:conclusion}
% ============================================================

Developing methods to evaluate and monitor AI systems will continue to be critical as AI adoption spreads across our economy. We present AM-PPI as a prediction-powered inference method that can help address this need. AM-PPI jointly optimizes predictor subset selection, gold-standard label allocation, and prediction weights under a shared deployment budget to improve statistical precision by reducing estimator variance. Under the stated regularity and calibration assumptions, AM-PPI yields asymptotically valid confidence intervals.

Across 47 dataset--cost--budget scenarios, AM-PPI variants achieved the best average interval-width ranks among 15 methods. Their mean relative interval-width reduction was approximately 20\% compared with classical estimation, with average empirical coverage close to the nominal 90\% level. The budget sweeps show that querying the most accurate predictor for every instance need not produce the most precise population estimate.

The advantage of AM-PPI over other prediction-powered methods depends on the quality of the optimization solution obtained during calibration. A priority for future work is to determine when and how to recalibrate under distribution shift. Our initial covariate-shift stress tests show that recalibrating on a smaller fresh sample did not consistently improve interval width or empirical coverage relative to stale calibration (Appendix~\ref{sec:rebuttal-shift}). The semi-synthetic clinical study motivates post-deployment monitoring as a potential application, with prospective validation needed to assess its utility in clinical deployment. More broadly, jointly budgeting predictor queries and outcome collection offers a way to use scarce human evaluation effort more effectively as AI use expands.

\section*{Acknowledgments}
We thank Paul Bartholomew, Scott Rosen, Shreya Kurdukar, and Yifty Eisenberg for helpful discussions. This work was supported by the MITRE Independent Research and Development Program.
Approved for Public Release; Distribution Unlimited. Public Release Case Number 26-2603. \copyright{} 2026 The MITRE Corporation. ALL RIGHTS RESERVED.

% Keep the main-text figures and tables ahead of the references and appendices.
\clearpage
\bibliographystyle{plainnat}
\bibliography{references}

\clearpage
\appendix

\section{Proof of Theorem~\ref{thm:clt} (Asymptotic Normality)}
\label{app:proof_clt}

\subsection*{Proof Overview}

The proof exploits the fact that all nuisance parameters $(\hat\pi, \hat\lambda, \hat I)$ are estimated from the burn-in sample and are therefore \emph{fixed constants} conditional on the burn-in data. Here we make the following assumptions:
\begin{enumerate}
    \item[(i)] \textbf{Moment bounds:} $\E[Y^4] < \infty$ and $\E[\|f_I(X)\|^4] < \infty$ for all $I \in \mathcal{I}$.
    \item[(ii)] \textbf{Nuisance convergence:} As the burn-in size $N \to \infty$,
    \begin{align*}
        \|\hat\pi - \pi^*\|_{L_2(P_X)} &\plim 0, \\
        \|\hat\lambda_I - \lambda^*_I\|_2 &\plim 0 \quad \text{for all } I \in \mathcal{I}, \\
        P_X(\hat I(X) \neq I^*(X)) &\plim 0.
    \end{align*}
    \item[(iii)] \textbf{Sampling lower bound:} $\pi^*(x) \geq \pi_{\min} > 0$ and $\hat\pi(x) \geq \pi_{\min} > 0$ for all $x$, where the latter holds either by assumption or by algorithmic clipping.
\end{enumerate}

This permits a direct application of the classical central limit theorem (CLT) to the increments:
\begin{enumerate}
    \item Compute the conditional mean over the given burn-in data and show it equals $\theta^*$ for any nuisance value with $\hat{\pi}>0$.
    \item Apply the Lindeberg--L\'evy CLT conditionally on the burn-in.
    \item Show the conditional variance converges in probability to $V$.
    \item Show unconditional convergence.
\end{enumerate}

\subsection*{Step 1: Conditional Mean}

Let $\cF_N$ denote the $\sigma$-algebra generated by the burn-in data. Conditional on $\cF_N$, the nuisance parameters $(\hat\pi, \hat\lambda_I, \hat I)$ are deterministic, and the deployment data $(X_i, Y_i)_{i=1}^n$ are i.i.d.\ with $\xi_i \sim \mathrm{Bern}(\hat\pi(X_i))$ drawn independently. The increments
\[
\Delta_i = \hat\lambda_{\hat I}^\top f_{\hat I}(X_i) + \frac{\bigl(Y_i - \hat\lambda_{\hat I}^\top f_{\hat I}(X_i)\bigr)\,\xi_i}{\hat\pi(X_i)}
\]
are therefore conditionally i.i.d.\ given $\cF_N$.

We compute the conditional mean of the summands given the burn-in and show it equals $\theta^*$ for any nuisance values.

We first apply linearity of expectation:
\begin{equation*}
\mathbb{E}[\Delta_i \mid X_i, \mathcal{F}_N] = \mathbb{E}\big[\hat\lambda_{\hat I}^\top f_{\hat I}(X_i) \mid X_i, \mathcal{F}_N\big] + \mathbb{E}\bigg[\frac{(Y_i - \hat\lambda_{\hat I}^\top f_{\hat I}(X_i))\,\xi_i}{\hat\pi(X_i)} \,\bigg|\, X_i, \mathcal{F}_N\bigg].
\end{equation*}

Conditional on $\mathcal{F}_N$, the functions $\hat\pi$, $\hat\lambda$, and $\hat I$ are deterministic; evaluating at the fixed point $X_i$ makes $\hat\pi(X_i)$, $\hat\lambda_{\hat I}$, and $f_{\hat I}(X_i)$ deterministic constants. The first term is therefore equal to $\hat\lambda_{\hat I}^\top f_{\hat I}(X_i)$, and $\hat\pi(X_i)$ pulls out of the second term:
\begin{equation*}
\mathbb{E}[\Delta_i \mid X_i, \mathcal{F}_N] = \hat\lambda_{\hat I}^\top f_{\hat I}(X_i) + \frac{1}{\hat\pi(X_i)}\,\mathbb{E}\big[(Y_i - \hat\lambda_{\hat I}^\top f_{\hat I}(X_i))\,\xi_i \mid X_i, \mathcal{F}_N\big].
\end{equation*}

Since $\xi_i$ is independent of $Y_i$ conditional on $X_i$, the expectation factors:
\begin{equation*}
\mathbb{E}\big[(Y_i - \hat\lambda_{\hat I}^\top f_{\hat I}(X_i))\,\xi_i \mid X_i, \mathcal{F}_N\big] = \mathbb{E}\big[Y_i - \hat\lambda_{\hat I}^\top f_{\hat I}(X_i) \mid X_i\big] \cdot \mathbb{E}[\xi_i \mid X_i, \mathcal{F}_N].
\end{equation*}

Using $\mathbb{E}[\xi_i \mid X_i, \mathcal{F}_N] = \hat\pi(X_i)$ and substituting back:
\begin{equation*}
\mathbb{E}[\Delta_i \mid X_i, \mathcal{F}_N] = \hat\lambda_{\hat I}^\top f_{\hat I}(X_i) + \mathbb{E}\big[Y_i - \hat\lambda_{\hat I}^\top f_{\hat I}(X_i) \mid X_i\big] \cdot \frac{\hat\pi(X_i)}{\hat\pi(X_i)}.
\end{equation*}

The $\hat\pi(X_i)$ factors cancel; this is the cancellation that makes this form structurally unbiased for any nuisance values with $\hat\pi(x) > 0$. The remaining prediction terms cancel as well:
\begin{equation}
\mathbb{E}[\Delta_i \mid X_i, \mathcal{F}_N] = \hat\lambda_{\hat I}^\top f_{\hat I}(X_i) + \mathbb{E}[Y_i \mid X_i] - \hat\lambda_{\hat I}^\top f_{\hat I}(X_i) = \mathbb{E}[Y_i \mid X_i]. \label{eq:conditional-mean}
\end{equation}

By the tower property:
\begin{equation*}
\mathbb{E}[\Delta_i \mid \mathcal{F}_N] =
\mathbb{E}[\mathbb{E}[\Delta_i \mid X_i,  \mathcal{F}_N]
 \mid \mathcal{F}_N]=
 \mathbb{E}[\mathbb{E}[Y_i \mid X_i ]
 \mid \mathcal{F}_N] = \mathbb{E}[Y_i] = \theta^*.
\end{equation*}

\subsection*{Step 2: Conditional CLT}

Since $\{\Delta_i\}_{i=1}^n$ are conditionally i.i.d.\ given $\mathcal{F}_N$ with mean $\theta^*$, we compute the conditional variance and invoke the CLT.

Define $\hat M_i = \hat\lambda_{\hat I}^\top f_{\hat I}(X_i)$ and $\hat R_i = Y_i - \hat M_i$. Then rewrite the summand as
\[
\Delta_i = \hat M_i + \hat R_i \cdot \frac{\xi_i}{\hat\pi(X_i)}.
\]
Conditional on $(X_i, Y_i, \mathcal{F}_N)$, the quantities $\hat\pi(X_i)$, $\hat I(X_i)$, $\hat M_i$, and $\hat R_i$ are deterministic, and the only remaining randomness is $\xi_i \sim \mathrm{Bern}(\hat\pi(X_i))$. Using $\mathbb{E}[\xi_i \mid X_i, Y_i, \mathcal{F}_N] = \hat\pi(X_i)$ (the labeling decision depends only on $X_i$) and $\mathrm{Var}(\xi_i \mid X_i, Y_i, \mathcal{F}_N) = \hat\pi(X_i)(1 - \hat\pi(X_i))$ (considering that $\text{Var}(aZ) = a^2\,\text{Var}(Z)$)  :
\begin{align*}
\mathbb{E}[\Delta_i \mid X_i, Y_i, \mathcal{F}_N] &= \hat M_i + \hat R_i \cdot \frac{\hat\pi(X_i)}{\hat\pi(X_i)} = Y_i, \\
\mathrm{Var}(\Delta_i \mid X_i, Y_i, \mathcal{F}_N) &= \left(\frac{\hat R_i}{\hat\pi(X_i)}\right)^2 \cdot \hat\pi(X_i)(1 - \hat\pi(X_i)) = \hat R_i^2 \cdot \frac{1 - \hat\pi(X_i)}{\hat\pi(X_i)}.
\end{align*}
Applying the law of total variance and using the independence of $(X_i, Y_i)$ from $\mathcal{F}_N$ to simplify $\mathrm{Var}(Y_i \mid \mathcal{F}_N) = \mathrm{Var}(Y)$:
\begin{equation}\label{eq:VN}
V_N := \mathrm{Var}(\Delta_i \mid \mathcal{F}_N) = \mathrm{Var}(Y) + \mathbb{E}\!\left[\frac{\hat R_i^2\,(1 - \hat\pi(X_i))}{\hat\pi(X_i)} \;\bigg|\; \mathcal{F}_N\right].
\end{equation}
The first term is the irreducible variance of the outcome. The second term is the additional variance introduced by the labeling lottery: instances with large residuals $\hat R_i^2$ and low sampling probabilities $\hat\pi(X_i)$ contribute the most variance, which is precisely the trade-off that the optimal sampling rule $\pi^*(x) \propto \sqrt{r_I(x)}$ is designed to minimize. Under assumption~(i), $\mathbb{E}[Y^2] < \infty$ and $\mathbb{E}[\|f_I(X)\|^2] < \infty$. For each finite realized calibration weight vector, this gives $\mathbb{E}[\hat M_i^2 \mid \mathcal{F}_N] < \infty$ and hence $\mathbb{E}[\hat R_i^2 \mid \mathcal{F}_N] < \infty$. Under assumption~(iii), $1/\hat\pi(X_i) \leq 1/\pi_{\min}$. Therefore $V_N < \infty$.

Since $\{\Delta_i\}$ are conditionally i.i.d.\ given $\mathcal{F}_N$ with finite conditional variance $V_N$, the Lindeberg--L\'evy CLT applied conditionally on $\mathcal{F}_N$ gives:
\begin{equation}\label{eq:cond_clt}
\sqrt{n}\bigl(\hat\theta - \theta^*\bigr) \mid \mathcal{F}_N \;\xrightarrow{d}\; \mathcal{N}(0, V_N).
\end{equation}
For joint limits in $n$ and $N$, assumptions~(i)--(iii) also give conditional fourth moments bounded in probability, yielding the conditional Lindeberg condition. If the limiting variance is zero, conditional Chebyshev bounds give the corresponding degenerate limit.

\subsection*{Step 3: Convergence of Conditional Variance}

We show $V_N \plim V$. From Step~2:
\[
V_N = \Var(Y) + \E\!\left[\frac{\hat R_i^2\,(1 - \hat\pi(X_i))}{\hat\pi(X_i)} \;\middle|\; \cF_N\right].
\]
At oracle values, the asymptotic variance defined in the statement of Theorem~\ref{thm:clt} (with $R_i^* = Y_i - \lambda_{I^*}^{*\top} f_{I^*}(X_i)$) is exactly
\[
V = \Var(Y) + \E\!\left[R_i^{*2}\!\left(\frac{1}{\pi^*(X_i)} - 1\right)\right] = \Var(Y) + \E\!\left[\frac{R_i^{*2}\,(1 - \pi^*(X_i))}{\pi^*(X_i)}\right],
\]
which already matches the form of $V_N$ derived in Step~2: both have a leading $\Var(Y)$ term plus a residual-variance term involving $1/\pi - 1$. We can therefore proceed directly to comparing the residual terms without any further rewriting of $V$.

With both quantities in the same form, the difference is:
\begin{equation}\label{eq:VN_diff}
V_N - V = \E\!\left[\frac{\hat R_i^2\,(1 - \hat\pi(X_i))}{\hat\pi(X_i)} \;\middle|\; \cF_N\right] - \E\!\left[\frac{R_i^{*2}\,(1 - \pi^*(X_i))}{\pi^*(X_i)}\right].
\end{equation}
Define $w(x, \pi) = (1 - \pi(x))/\pi(x)$ for shorthand. We decompose the difference into a sampling-weight term and a residual term:
\begin{equation}\label{eq:AB_decomp}
V_N - V = \underbrace{\E\!\bigl[\hat R_i^2\,(w(X_i, \hat\pi) - w(X_i, \pi^*))\bigr]}_{(A)} + \underbrace{\E\!\bigl[(\hat R_i^2 - R_i^{*2})\,w(X_i, \pi^*)\bigr]}_{(B)},
\end{equation}
where all expectations are over $(X_i, Y_i) \sim P_{X,Y}$ with nuisance parameters held fixed (i.e., conditional on $\cF_N$).

\medskip
\noindent\textit{Bounding (A):}
Since $w(x, \pi) = (1-\pi)/\pi$, we have
\[
|w(X_i, \hat\pi) - w(X_i, \pi^*)| = \left|\frac{1}{\hat\pi(X_i)} - \frac{1}{\pi^*(X_i)}\right| \cdot |1| = \frac{|\pi^*(X_i) - \hat\pi(X_i)|}{\hat\pi(X_i)\,\pi^*(X_i)} \leq \frac{|\hat\pi(X_i) - \pi^*(X_i)|}{\pi_{\min}^2}.
\]
Therefore:
\begin{align*}
|A| &\leq \frac{1}{\pi_{\min}^2}\,\E\!\bigl[\hat R_i^2 \cdot |\hat\pi(X_i) - \pi^*(X_i)| \,\bigm|\, \cF_N\bigr] \\
&\leq \frac{1}{\pi_{\min}^2}\,\bigl(\E[\hat R_i^4 \mid \cF_N]\bigr)^{1/2}\,\|\hat\pi - \pi^*\|_{L_2(P_X)}, \tag{Cauchy--Schwarz}
\end{align*}
where the conditional fourth moment $\E[\hat R_i^4 \mid \cF_N]$ is bounded in probability: by assumption~(ii), $\hat\lambda$ converges and is therefore stochastically confined to a compact set, so assumption~(i) gives that $\E[\hat R_i^4 \mid \cF_N]$ is bounded in probability. The second factor vanishes by assumption~(ii). The product of a bounded-in-probability term with one that vanishes in probability vanishes in probability, so $A \plim 0$.

\medskip
\noindent\textit{Bounding (B):}
Since $w(X_i, \pi^*) \leq 1/\pi_{\min}$, we have $|B| \leq \frac{1}{\pi_{\min}}\,\E[|\hat R_i^2 - R_i^{*2}|]$. We decompose by the event $\{\hat I(X_i) = I^*(X_i)\}$:
\[
|B| \leq \frac{1}{\pi_{\min}}\,\E\!\bigl[|\hat R_i^2 - R_i^{*2}| \cdot \1\{\hat I(X_i) = I^*(X_i)\}\bigr] + \frac{1}{\pi_{\min}}\,\E\!\bigl[|\hat R_i^2 - R_i^{*2}| \cdot \1\{\hat I(X_i) \neq I^*(X_i)\}\bigr] =: B_1 + B_2.
\]
For $B_1$: on the event $\{\hat I(X_i) = I^*(X_i)\}$, both residuals use the same predictor subset, so $\hat R_i - R_i^* = (\lambda^*_{I^*} - \hat\lambda_{I^*})^\top f_{I^*}(X_i)$. Factoring the difference of squares:
\[
|\hat R_i^2 - R_i^{*2}| = |\hat R_i - R_i^*| \cdot |\hat R_i + R_i^*| = |(\hat\lambda_{I^*} - \lambda^*_{I^*})^\top f_{I^*}(X_i)| \cdot |2Y_i - (\hat\lambda_{I^*} + \lambda^*_{I^*})^\top f_{I^*}(X_i)|.
\]
By Cauchy--Schwarz, $B_1 \leq C_N\max_I\|\hat\lambda_I - \lambda^*_I\|$, where $C_N$ is bounded in probability by assumptions~(i)--(ii). Since the candidate family is finite, assumption~(ii) gives $\max_I\|\hat\lambda_I - \lambda^*_I\| \plim 0$, so $B_1 \plim 0$.

For $B_2$: by Cauchy--Schwarz,
\[
B_2 \leq \frac{1}{\pi_{\min}}\,\bigl(\E[(\hat R_i^2 - R_i^{*2})^2]\bigr)^{1/2}\,\bigl(P_X(\hat I(X) \neq I^*(X))\bigr)^{1/2}.
\]
The first factor is bounded in probability by the moment assumptions and weight convergence; the second vanishes by assumption~(ii). Therefore $B \plim 0$, and hence $V_N \plim V$.

\subsection*{Step 4: Unconditional Convergence}

We have established:
\begin{enumerate}
    \item $\sqrt{n}(\hat\theta - \theta^*) \mid \mathcal{F}_N \xrightarrow{d} \mathcal{N}(0, V_N)$ \hfill (Step~2, Eq.~\eqref{eq:cond_clt})
    \item $V_N \xrightarrow{p} V$ \hfill (Step~3)
\end{enumerate}
We now lift to unconditional convergence. Let $Z_n = \sqrt{n}(\hat\theta - \theta^*)$. The conditional convergence in~(1) implies that the conditional characteristic function converges:
\[
\mathbb{E}\bigl[e^{itZ_n} \mid \mathcal{F}_N\bigr] \xrightarrow{p} e^{-t^2 V_N / 2}
\]
for each $t \in \mathbb{R}$. Since $V_N \xrightarrow{p} V$ by~(2), we have $e^{-t^2 V_N/2} \xrightarrow{p} e^{-t^2 V/2}$. Taking unconditional expectations via the tower property:
\[
\mathbb{E}\bigl[e^{itZ_n}\bigr] = \mathbb{E}\bigl[\mathbb{E}[e^{itZ_n} \mid \mathcal{F}_N]\bigr] \to e^{-t^2 V/2},
\]
where boundedness by~1 converts convergence in probability of the conditional characteristic functions into convergence of their expectations. The right-hand side is the characteristic function of $\mathcal{N}(0, V)$, so by L\'evy's continuity theorem:
\[
\sqrt{n}(\hat\theta - \theta^*) \xrightarrow{d} \mathcal{N}(0, V). \qed
\]

This completes the proof. \qed

\section{Proof of Corollary~\ref{cor:ci} (Valid Confidence Intervals)}
\label{app:proof_ci}
We can now use the result from Theorem~\ref{thm:clt} to prove Corollary~\ref{cor:ci}.

\begin{proof}
The proof requires showing $\hat\sigma^2 \plim V$, after which the result follows from Theorem~\ref{thm:clt} by Slutsky's theorem. The conditional framework makes this straightforward.

\medskip
\noindent\textbf{Step 1: Conditional law of large numbers.}

Conditional on $\cF_N$, the increments $\{\Delta_i\}_{i=1}^n$ are i.i.d.\ with
\[
\E[\Delta_i \mid \cF_N] = \theta^*, \qquad
\Var(\Delta_i \mid \cF_N) = V_N,
\]
as established in the proof of Theorem~\ref{thm:clt}.

Rewrite the variance estimator as
\[
\hat\sigma^2 = \frac{1}{n}\sum_{i=1}^n (\Delta_i)^2 - \hat\theta^2.
\]
Since the $\{\Delta_i\}$ are conditionally i.i.d.\ and $\E[(\Delta_i)^4 \mid \cF_N] < \infty$ (verified in the Theorem~\ref{thm:clt} proof), the conditional law of large numbers (LLN) gives:
\begin{align}
\frac{1}{n}\sum_{i=1}^n (\Delta_i)^2 \;\bigg|\; \cF_N &\;\plim\; \E[(\Delta_i)^2 \mid \cF_N] = V_N + (\theta^*)^2, \label{eq:lln_sq} \\[4pt]
\hat\theta^2 \;\bigg|\; \cF_N &\;\plim\; (\theta^*)^2. \label{eq:thetasq}
\end{align}
Equation~\eqref{eq:thetasq} follows because $\hat\theta \mid \cF_N \plim \theta^*$ (immediate from the conditional CLT with $\sqrt{n}$-scaling, or directly from the conditional LLN applied to $\hat\theta$ itself). Subtracting:
\begin{equation}\label{eq:cond_var_consist}
\hat\sigma^2 \;\bigg|\; \cF_N \;\plim\; V_N.
\end{equation}

\medskip
\noindent\textbf{Step 2: From conditional to unconditional consistency.}

From the proof of Theorem~\ref{thm:clt} (Step~3), we have $V_N \plim V$. Combining with \eqref{eq:cond_var_consist}:

For any $\varepsilon > 0$,
\begin{align*}
P(|\hat\sigma^2 - V| > \varepsilon)
&\leq P(|\hat\sigma^2 - V_N| > \varepsilon/2) + P(|V_N - V| > \varepsilon/2).
\end{align*}
For the first term, condition on $\cF_N$:
\[
P(|\hat\sigma^2 - V_N| > \varepsilon/2) = \E\!\bigl[P(|\hat\sigma^2 - V_N| > \varepsilon/2 \mid \cF_N)\bigr] \to 0,
\]
since the conditional fourth-moment bound gives a conditional deviation probability tending to zero in probability, and boundedness by~1 gives convergence of its expectation. The second term vanishes by $V_N \plim V$. Therefore
\[
\hat\sigma^2 \plim V.
\]
\medskip
\noindent\textbf{Step 3: Slutsky's theorem.}
From Theorem~\ref{thm:clt}: $\sqrt{n}(\hat\theta - \theta^*) \dlim N(0, V)$. From Step~2: $\hat\sigma^2 \plim V$, hence $\hat\sigma \plim \sqrt{V}>0$ under the corollary's nondegeneracy assumption. By Slutsky's theorem:
\[
\frac{\sqrt{n}(\hat\theta - \theta^*)}{\hat\sigma} \dlim N(0, 1).
\]
Therefore
\[
\lim_{n\to\infty} P(\theta^* \in C_\alpha) = \lim_{n\to\infty} P\!\left(\left|\frac{\sqrt{n}(\hat\theta - \theta^*)}{\hat\sigma}\right| \leq z_{1-\alpha/2}\right) = 1 - \alpha.
\]
\end{proof}

\section{Derivation of Optimality Conditions}
\label{app:optimization}

We derive the conditional optimality conditions stated in Section~\ref{sec:method}. Each update is derived with the other quantities fixed. The joint objective is not jointly convex in the sampling probabilities and prediction weights, and these conditions characterize the individual updates without certifying the solution of the joint problem. In practice, refitting the uncertainty surrogates between updates also requires checking calibration convergence. We use the shorthand $r_I(x) = \E[(Y - \lambda_I^\top f_I(x))^2 \mid X = x]$ for the conditional residual second moment throughout this section. The Lagrangian for the constrained optimization~\eqref{eq:opt} is
\[
\mathcal{L} = \frac{1}{n}\,\E\!\left[r_{I(X)}(X)\!\left(\frac{1}{\pi_{I(X)}(X)} - 1\right)\right] + \mu\!\left(\E\!\left[c_{I(X)} + \pi_{I(X)}(X)\,c_{\text{label}}\right] - b\right).
\]

\subsection{Optimal sampling probability \texorpdfstring{$\pi^*$}{pi*}}

The $\pi$-dependent terms of the Lagrangian are
\[
\mathcal{L}_\pi = \frac{1}{n}\,\E\!\left[\frac{r_I(X)}{\pi_I(X)}\right] + \mu\,\E\!\left[\pi_I(X)\,c_{\text{label}}\right],
\]
where we used the tower property to write $\E[(Y - \lambda_I^\top f_I)^2 / \pi_I(X)] = \E[r_I(X)/\pi_I(X)]$ since $\pi_I$ is a function of $X$ alone. Perturbing $\pi_I(x) \to \pi_I(x) + \varepsilon\,\eta(x)$ and differentiating at $\varepsilon = 0$:
\[
\frac{d}{d\varepsilon}\mathcal{L}_\pi\bigg|_{\varepsilon=0} = \E\!\left[\left(-\frac{r_I(X)}{n\,\pi_I(X)^2} + \mu\,c_{\text{label}}\right)\eta(X)\right] = 0.
\]
Since this must hold for all perturbations $\eta$, the integrand vanishes pointwise, giving $r_I(x)/(n\,\pi_I(x)^2) = \mu\,c_{\text{label}}$. Solving yields Eq.~\eqref{eq:pi_star}:
\[
\pi_I^*(x) = \sqrt{\frac{r_I(x)}{n\,\mu\,c_{\text{label}}}},
\]
where $\mu$ is chosen to satisfy the budget constraint. The unconstrained pointwise minimizer $\pi_I^*$ may exceed $1$ on instances with sufficiently large $r_I(x)$. The feasible projection onto $(0,1]$ is the clipped probability $\hat\pi_I(x) = \min(1, \pi_I^*(x))$ (Eq.~\eqref{eq:pi_clipped}); on the deterministic set $\{\hat\pi_I = 1\}$ the variance contribution $r_I(x)(1/\hat\pi_I - 1)$ vanishes, recovering exact PPI on those instances. We use $\hat\pi_I$ in all downstream substitutions.

\subsection{Combination weights for fixed routing and sampling}

For a fixed routing rule and fixed sampling probabilities, the part of the Lagrangian depending on $\lambda_I$ is
\[
\frac{1}{n}\E\!\left[\1\{I(X)=I\}\left(\frac{1}{\hat\pi_I(X)}-1\right)
  (Y-\lambda_I^\top f_I(X))^2\right].
\]
Differentiating gives Eq.~\eqref{eq:lambda_star}. This is a convex weighted least-squares problem because its weights are nonnegative; the solution is unique if the weighted design moment matrix is positive definite. Instances outside the routed subset, and instances labeled with probability one, have zero weight.

The default implementation instead fits each candidate subset on all calibration cases with weights $1/\hat\pi_I-1$, to retain information when a subset receives few routed cases. This all-case update is a practical surrogate and does not satisfy the routed condition in general. The optional indicator-fit variant uses the routed weights directly.

\subsection{Lagrange multiplier \texorpdfstring{$\mu$}{mu}}

The budget constraint with the deployed clipped probability is
\[
\E\!\left[c_{I(X)} + c_{\text{label}}\,\hat\pi_{I(X)}(X)\right] = b
\qquad \text{(Eq.~\eqref{eq:mu_condition})},
\]
when the expected-cost constraint is active and equality is attainable. With fixed residual functions, expected cost is non-increasing in $\mu$; bisection retains the feasible endpoint if discrete routing changes cause a jump. Substituting the unclipped $\pi_I^*$ (assuming no instance clips and a fixed predictor subset is used) yields the algebraic identity
\[
\E[c_I] + \sqrt{\frac{c_{\text{label}}}{n\mu}}\,\E\!\left[\sqrt{r_I(X)}\right] = b,
\]
which solves to the closed form
\[
\mu = \frac{c_{\text{label}}}{n}\!\left(\frac{\E\!\left[\sqrt{r_I(X)}\right]}{b - \E[c_I]}\right)^{\!2},
\]
requiring $b > \E[c_I]$ (the per-instance budget exceeds the expected prediction cost). This closed form is exact only when no instance clips and a single fixed subset is used. In the general AM-PPI setting, the optimal subset $I^*(x;\mu)$ depends on $\mu$ via Eq.~\eqref{eq:I_star}, so the right-hand side $\E[\sqrt{r_{I^*(X;\mu)}(X)}]$ depends on $\mu$ and the expression is a fixed-point equation rather than an algebraic closed form. Algorithm~\ref{alg:ampi_calib} therefore searches for a feasible multiplier numerically, allowing slack when routing changes discontinuously.

\subsection{Optimal predictor assignment \texorpdfstring{$I^*$}{I*}}

The subset assignment $I$ is a function of $X$. Conditioning on $X = x$, and writing the cost in terms of the deployed clipped probability $\hat\pi_I$, the per-instance Lagrangian cost of choosing predictor subset $I$ is
\[
\ell_I(x) = \frac{r_I(x)}{n}\!\left(\frac{1}{\hat\pi_I(x)} - 1\right) + \mu\,c_I + \mu\,c_{\text{label}}\,\hat\pi_I(x),
\]
which is the form used in Eq.~\eqref{eq:I_star}. The optimal assignment selects the subset minimizing $\ell_I(x)$ at each $x$.

\paragraph{Two regimes.} On the unclipped set $\{\hat\pi_I(x) < 1\}$, substituting $\hat\pi_I = \pi_I^* = \sqrt{r_I/(n\mu c_{\text{label}})}$ gives
\begin{align*}
\frac{r_I(x)}{n}\!\left(\frac{1}{\pi_I^*(x)} - 1\right) &= \frac{\sqrt{\mu c_{\text{label}}}}{\sqrt{n}}\,\sqrt{r_I(x)} - \frac{r_I(x)}{n}, \\
\mu\,c_{\text{label}}\,\pi_I^*(x) &= \frac{\sqrt{\mu c_{\text{label}}}}{\sqrt{n}}\,\sqrt{r_I(x)},
\end{align*}
so the cost simplifies to
\[
\ell_I(x) = \frac{2\sqrt{\mu c_{\text{label}}}}{\sqrt{n}}\,\sqrt{r_I(x)} - \frac{r_I(x)}{n} + \mu\,c_I,
\]
which is monotonically increasing in $\sqrt{r_I(x)}$ on $\{\hat\pi_I < 1\}$ (i.e., for $\sqrt{r_I(x)} \le \sqrt{n\mu c_{\text{label}}}$). On the deterministic set $\{\hat\pi_I(x) = 1\}$, the variance contribution vanishes and $\ell_I(x) = \mu(c_I + c_{\text{label}})$, a constant in $r_I$. Combining the two regimes, the unified expression in Eq.~\eqref{eq:I_star} is non-decreasing in $r_I(x)$ across the entire range.

\section{Two-population model advantage analysis}
\label{app:advantage}

To develop intuition around \emph{when} two predictors with differing costs and strengths provide a statistical advantage over a single predictor, we derive a closed-form variance ratio between ASI~\citep{zrnic2024active} and AM-PPI under a simplified two-population model. Consider $k = 2$ predictors with costs $c_1 < c_2$. A fraction $p$ of instances are ``easy'' (all predictors achieve small residual variance $r_e$) and a fraction $1 - p$ are ``hard'' (the cheap predictor has residual variance $r_h$ on hard examples, whereas the expensive predictor achieves $r_h(1 - \delta)$ for $\delta \in [0,1]$). In this example, AM-PPI has access to two predictors while ASI commits to a single predictor $j$ for all instances under each test scenario. Here we define the average root residual as $S_j = p \sqrt{r_{j,e}} + (1 - p) \sqrt{r_{j,h}}$. For this illustration, the singleton uses unit prediction weight and oracle uncertainty $\sqrt{r_j(x)}$, without clipping or uniform mixing; this idealization differs from the practical ASI baselines. Its variance is:

\begin{equation}
    \Var(\thetahat_{\text{ASI},j}) = \frac{1}{n}\,\Var(Y) + \frac{1}{n}\!\left[\frac{S_j^2 \cdot c_{\text{label}}}{b - c_j} - T_j\right],
    \label{eq:var_asi}
\end{equation}
where $T_j := \E[r_j] = p\,r_{j,e} + (1-p)\,r_{j,h}$. The variance consists of a tradeoff between the residual $S_j$ and budget-related factors. AM-PPI navigates this tradeoff by routing easy instances to the cheap predictor and hard instances to the expensive one. The average model cost is $C_{\text{route}} = p\,c_1 + (1-p)\,c_2$, and the variance is:
\begin{equation}
    \Var(\hat\theta_{\text{AM-PPI}}) = \frac{1}{n}\,\Var(Y) + \frac{1}{n}\!\left[\frac{S_2^2 \cdot c_{\text{label}}}{b - C_{\text{route}}} - T_2\right].
    \label{eq:var_mixed}
\end{equation}

Note that the AM-PPI residual in the two-population model is equal to $S_2$; the population of easy examples $p$ are routed to the easy model with variance $r_e$ and likewise for the hard instances yielding an average root residual of $p\sqrt{r_e}+(1-p)\sqrt{r_h(1-\delta)} = S_2$. Since $C_{\text{route}} < c_2$, AM-PPI retains more budget for labels than ASI with the expensive predictor. In the tight-budget regime where active sampling dominates ($\pi^* \ll 1$, so $\E[r/\pi^*] \gg \E[r]$), the $\Var(Y) - T_j$ terms in Eqs.~\eqref{eq:var_asi}--\eqref{eq:var_mixed} are sub-leading relative to the budget-amplified residual term, and the variance ratio simplifies to a pure budget-savings form. We now analyze this ratio of AM-PPI relative to ASI.

\paragraph{Case 1: ASI uses the expensive predictor.} The variance ratio (ASI/AM-PPI) simplifies to a pure budget-savings ratio. With a normalized model budget $\varphi = c_2/b$ and $c_1 \ll c_2$:

\begin{equation}
    R_2 \;\approx\; 1 + \frac{p \, \varphi}{1 - \varphi}.
    \label{eq:R2_universal}
\end{equation}

$R_2$ depends only on the fraction of easy instances $p$ in the dataset and the budget fraction $\varphi$. In this case, AM-PPI improves upon ASI $(R_2 > 1)$ as the number of easy examples in the dataset and the cost of the expensive model increases. Figure~\ref{fig:cases} (left) visualizes the CI width reduction.

\paragraph{Case 2: ASI uses the cheap predictor.} Unlike Case 1, the variance ratio reflects a tradeoff between variance reduction due to model performance and model cost. Assuming $r_e\ll r_h$ and $c_1\ll c_2$, the ratio simplifies to:

\begin{equation}
    R_1 \;\approx\; \frac{1 - (1-p)\,\varphi}{1 - \delta}.
    \label{eq:R1_universal}
\end{equation}

AM-PPI improves upon ASI ($R_1 > 1$) precisely when $\delta > (1-p)\varphi$, i.e., the per-hard-instance variance gain exceeds the budget cost. Figure~\ref{fig:cases} (right) visualizes this region at $p = 0.5$; from Eq.~\ref{eq:R1_universal}, increasing $p$ shrinks the $(1-p)\varphi$ threshold and further expands the region in which AM-PPI improves upon ASI.

\begin{figure}[H]
    \centering
    \includegraphics[width=\columnwidth]{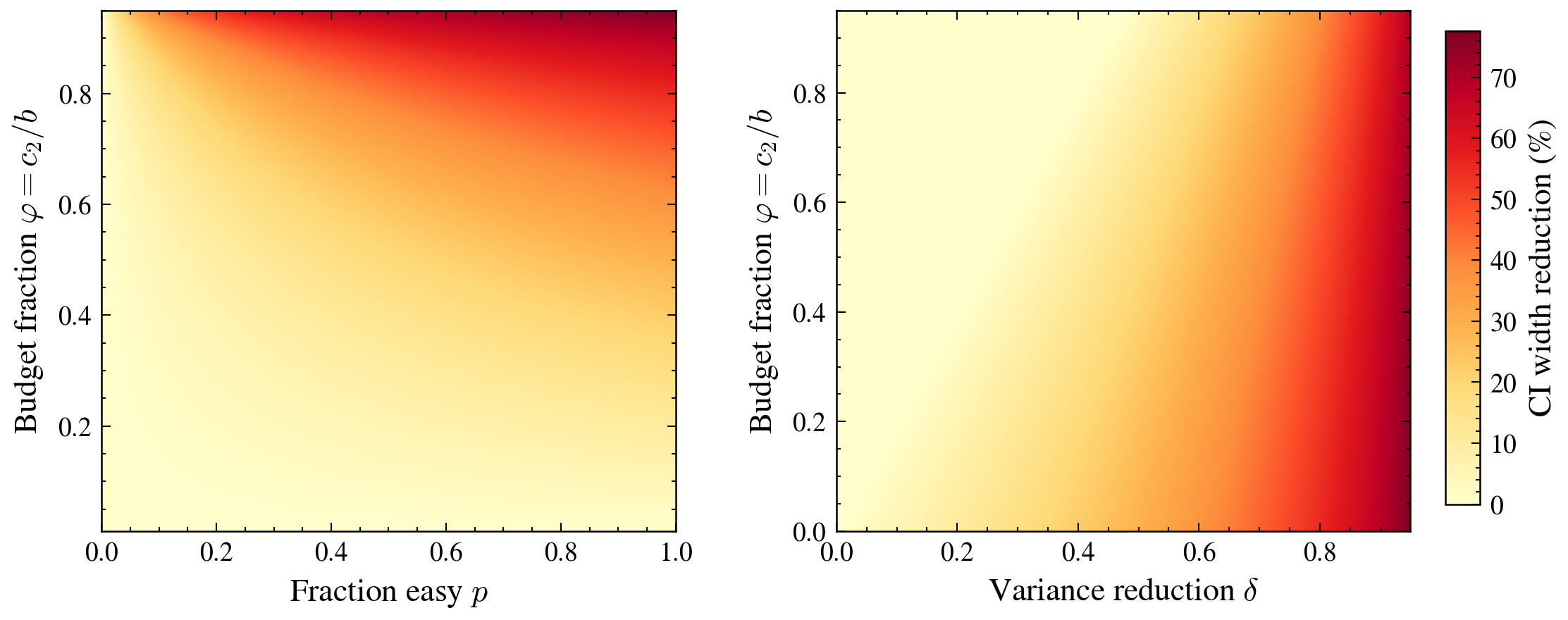}
    \caption{CI width reduction (\%) under the two-population model. \textbf{Left:} Case~1 (AM-PPI versus ASI using the expensive predictor) depends only on the fraction of easy examples $p$ and $\varphi = c_2/b$. \textbf{Right:} Case~2 at $p = 0.5$ (AM-PPI versus ASI using the cheap predictor) under $c_1 \ll c_2$ and $r_e \ll r_h$; AM-PPI improves upon ASI where $\delta > (1-p)\varphi$. Negative reductions are displayed as zero, so zero color does not distinguish parity from a loss.}
    \label{fig:cases}
\end{figure}

\section{Experimental details}
\label{app:exp_details}

This section contains details regarding the experiments described in Section~\ref{sec:experiments}. 

\paragraph{Data splits.} Separate datasets are used for predictor training, inference-policy calibration, and final evaluation. Synthetic regression uses a 60/20/20 split, MIMIC a 40/30/30 split, Adult uses a 50/25/25 split, and Bank Marketing uses a 50/25/25 split. All splits preserve observed outcome frequencies without balancing. Splits are at the row level. MIMIC original/perturbed records can cross splits, so those results do not establish generalization to independent patients. Adult targets the unweighted empirical record distribution and Bank targets the historical campaign mixture.

\paragraph{MIMIC predictor uncertainty.} The cached MIMIC predictors provide a binary prediction and an ordinal certainty score $s\in\{1,2,3\}$ for each of the $2{,}551$ rows. For the ASI baselines, we map certainty scores $1,2,3$ to probabilities $0.6,0.8,1.0$ when the predicted class is one and to $0.4,0.2,0.0$ when it is zero. These mapped probabilities are used in both the ASI estimator and its uncertainty rule $u(x)=\min\{f(x),1-f(x)\}$. The sampling mixture parameter $\tau$ is tuned on the calibration split. 

\paragraph{Number of trials.} Table~\ref{tab:grand_overall} is across 1,175 trials. Figure~\ref{fig:experiments} curves use $2{,}000$ trials for synthetic regression and $1{,}000$ each for MIMIC, Adult, and Bank. Each trial resamples test rows with replacement and draws labels independently, holding calibration fixed. Coverage is measured against the original test-set mean and therefore assesses the empirical distribution. Per-trial widths and coverage indicators yield the mean curves and $\pm1$ SEM bands in Figure~\ref{fig:experiments}.

\paragraph{Per-method viability cutoff.} The main plots use a minimum realized label count of $20$ for MIMIC, $10$ for Adult and Bank, and $7$ for synthetic regression. Individual trials below this cutoff are excluded, and a budget is displayed only if enough trials remain. A method must also meet the estimated affordability threshold $B\ge c_I n_{\mathrm{test}}+n_{\min}c_{\mathrm{label}}$, using its fixed predictor cost for singleton methods and the mean of the cheap and expensive costs as an approximation for AM-PPI. These are reporting filters, not a guarantee that each Bernoulli draw acquires the minimum number of labels. Consequently, displayed coverage is conditional on the retained trials. 

\paragraph{ASI baseline.} ASI follows the authors' batch-mode notebooks~\citep{zrnic2024active}, with unit prediction weight $\lambda=1$ and $\hat\theta_{\mathrm{ASI}}=n^{-1}\sum_i[f(X_i)+(Y_i-f(X_i))\xi_i/\pi(X_i)]$. Write $q=(B-c_I n_{\mathrm{test}})/(c_{\mathrm{label}}n_{\mathrm{test}})$ for the remaining label-budget fraction. The policy is $\pi(x)=\mathrm{clip}((1-\tau)\hat\eta u(x)+\tau q,0,1)$, where $\hat\eta=q/\bar u$ is normalized separately on the calibration and deployment batches. For binary outcomes, predictor uncertainty is $u(x)=\min\{f(x),1-f(x)\}$; $\tau$ minimizes the calibration mean of $(Y-f)^2/\pi$ over $100$ equally spaced values in $[0.01,0.99]$. For synthetic regression, the reference census recipe fits predictor-specific absolute errors $|Y-f(X)|$ with XGBoost (\texttt{reg:absoluteerror}, depth $7$, learning rate $0.3$, $2{,}000$ rounds) and fixes $\tau=0.001$. This uncertainty fit is separate from the squared-residual surrogates of AM-PPI and never uses AM-PPI weights. Clipping the policy can leave part of the requested label budget unused.

\paragraph{Uncertainty model architecture.} Each $\hat u_I(x)$ is derived from a gradient boosting regressor (sklearn \texttt{GradientBoostingRegressor}) trained on the calibration split to predict the squared residual $(Y - \hat\lambda_I^\top f_I(X))^2$ from the raw covariates $X$, so that the fitted mean estimates the conditional residual second moment $r_I(x)$ directly and $\hat u_I(x)$ is the square root of the (nonnegativity-clipped) fit (using 11 lightweight summary features for MIMIC, 12 encoded attributes for Adult, and 19 for Bank). Hyperparameters are selected once by 3-fold cross-validation with negative mean absolute error as the score and reused during refinement. Synthetic regression uses $n_{\mathrm{estimators}}\in\{25,50,100\}$, depths $\{2,6,8\}$, and learning rates $\{0.05,0.1\}$; Adult uses estimators $\{25,50\}$, depths $\{2,3\}$, and learning rate $0.05$; Bank fixes 50 depth-2 trees and learning rate $0.05$. Calibration runs for at most $50$ sweeps and stops when relative parameter drift is below $10^{-3}$ or ceases improving and the multiplier is re-solved after the final uncertainty fit.

\paragraph{Adult features and predictor training.} The 48,842 source records retain missing categorical values as a separate category; category encoding and numeric median imputation are fitted on training rows only. We exclude the survey-weight field \texttt{fnlwgt} and the redundant education string. The demographic features contain age, education number, marital status, relationship, race, and sex. The employment features contain work class, occupation, capital gain, capital loss, hours per week, and native country. Both classifiers use learning rate $0.05$ and natural training frequencies. To keep the simulations manageable, we limited the calibration and test sets to $2{,}400$ rows each. The final grid contains 12 budgets equal to $0.005n_{\mathrm{test}}$ plus geometrically spaced label allowances from $0.025n_{\mathrm{test}}$ to $0.20n_{\mathrm{test}}$, giving $B\in[72,492]$.

\paragraph{Bank features and predictor training.} The cheap predictor uses age, job, marital status, education, credit default, housing loan, and personal loan. The expensive predictor adds contact context, campaign history, and economic indicators, with call duration excluded from both predictors and the router. Categorical encodings use training rows only and unseen categories map to a fixed sentinel. Both classifiers use learning rate $0.05$ and observed class frequencies, without training subsampling or probability distortion. The 12 final budgets equal $0.05n_{\mathrm{test}}$ plus geometrically spaced label allowances from $0.03n_{\mathrm{test}}$ to $0.35n_{\mathrm{test}}$, giving $B\in[823.76,4118.80]$.

\paragraph{Compute resources.} AM-PPI and ASI results were generated using a single Apple M3 Pro CPU. LLM predictions were generated using 4$\times$144\,GB H200 and 1$\times$96\,GB H100. 

\section{Comparisons}
\label{app:comparisons}

This appendix describes the additional methods included in the overall comparison in Table~\ref{tab:grand_overall}. Dataset-specific ranks, empirical coverage summaries, and paired Cascade comparisons appear in Appendix~\ref{sec:rebuttal-comparison}. 

\paragraph{AM-PPI-1 (expensive / cheap).} AM-PPI restricted to a single fixed predictor ($|\mathcal{I}| = 1$). These variants inherit the WLS combination weight $\lambda$ from Eq.~\eqref{eq:lambda_star} and the optimal clipped sampling rule of Eq.~\eqref{eq:pi_clipped}, using learned residual uncertainty and tuned weights. Comparing them with AM-PPI assesses the effect of allowing additional candidate subsets and routing; comparing them with ASI assesses the different weighting and uncertainty policies at a fixed predictor.

\paragraph{PPI and PPI++ (expensive / cheap).} Each variant queries one fixed predictor on every deployment instance and uses uniform label sampling from the remaining budget. PPI fixes its prediction weight at one and PPI++ tunes the weight~\citep{angelopoulos2023ppiplus}. Both predictor identities are retained when their query costs tie.

\paragraph{Intercept variants and Cascade.} Intercept variants fit $\lambda_0+\lambda_I^\top f_I(x)$ jointly in the weighted update. Cascade queries the cheap predictor on every instance, escalates a calibration-tuned fraction of the most uncertain cases to the expensive predictor, and then samples labels actively. Both query costs are charged on escalated instances.

\section{Additional experiments and calibration checks}
\label{app:additional_results}

\floatplacement{table}{H}
\subsection{Expanded comparisons}
\label{sec:rebuttal-comparison}
\paragraph{AM-PPI and ASI comparison details.} Comparisons per dataset with AM-PPI routing percentages are included here. Width SE is the sample standard deviation divided by the square root of the retained trial count; coverage SE is the binomial Monte Carlo error. A dash denotes unavailable archived diagnostics. The minimum labeling probability floor of $1\%$ (discussed in Appendix~\ref{app:ablation}) can cause overspending at very low budgets.

\begin{table}
\centering
\small
\caption{AM-PPI Monte Carlo precision, realized label counts, and calibration routing shares on Synthetic regression.}
\label{tab:diagnostics_regression}
\begin{adjustbox}{max width=\textwidth}
\begin{tabular}{llllllll}
\toprule
Budget & Trials & Width SE & Coverage SE & Labels (SD) & Cheap only (\%) & Expensive only (\%) & Both (\%) \\
\midrule
75.0 & 2000 & 0.0022 & 0.007 & 47.7 (6.8) & 100.0 & 0.0 & 0.0 \\
124.0 & 2000 & 0.00099 & 0.007 & 81.8 (9.1) & 98.1 & 1.9 & 0.0 \\
172.9 & 2000 & 0.00068 & 0.007 & 112.5 (10.3) & 94.0 & 6.0 & 0.0 \\
221.9 & 2000 & 0.0011 & 0.007 & 144.5 (11.8) & 88.3 & 11.7 & 0.0 \\
270.8 & 2000 & 0.0012 & 0.006 & 171.3 (12.9) & 80.4 & 19.6 & 0.0 \\
319.8 & 2000 & 0.00087 & 0.007 & 196.7 (13.4) & 73.5 & 26.5 & 0.0 \\
368.8 & 2000 & 0.00087 & 0.007 & 229.4 (14.6) & 67.8 & 32.2 & 0.0 \\
417.7 & 2000 & 0.0007 & 0.007 & 262.7 (15.6) & 62.9 & 37.1 & 0.0 \\
466.7 & 2000 & 0.00072 & 0.007 & 295.4 (16.2) & 57.6 & 42.4 & 0.0 \\
515.6 & 2000 & 0.00061 & 0.007 & 335.5 (17.4) & 54.3 & 45.7 & 0.0 \\
564.6 & 2000 & 0.00052 & 0.007 & 374.9 (17.8) & 50.9 & 49.1 & 0.0 \\
613.5 & 2000 & 0.00062 & 0.007 & 413.8 (18.2) & 48.1 & 51.9 & 0.0 \\
662.5 & 2000 & 0.00042 & 0.006 & 457.4 (18.9) & 46.7 & 52.7 & 0.6 \\
711.5 & 2000 & 0.00051 & 0.007 & 502.3 (19.7) & 45.3 & 53.0 & 1.6 \\
760.4 & 2000 & 0.00047 & 0.007 & 543.4 (20.5) & 43.5 & 54.8 & 1.7 \\
809.4 & 2000 & 0.0004 & 0.007 & 589.7 (20.8) & 42.4 & 55.9 & 1.7 \\
858.3 & 2000 & 0.00029 & 0.007 & 621.2 (21.1) & 37.0 & 61.2 & 1.8 \\
907.3 & 2000 & 0.00051 & 0.007 & 668.6 (20.9) & 36.9 & 61.2 & 1.9 \\
956.2 & 2000 & 0.00025 & 0.007 & 710.7 (21.7) & 33.9 & 64.1 & 2.0 \\
1005.2 & 2000 & 0.00023 & 0.007 & 742.4 (22.3) & 28.2 & 69.5 & 2.2 \\
1054.2 & 2000 & 0.00035 & 0.007 & 791.3 (21.3) & 28.6 & 69.3 & 2.0 \\
1103.1 & 2000 & 0.00035 & 0.007 & 834.0 (21.5) & 26.4 & 71.5 & 2.1 \\
1152.1 & 2000 & 0.00051 & 0.007 & 888.1 (21.9) & 28.0 & 70.0 & 2.0 \\
1201.0 & 2000 & 0.00043 & 0.007 & 941.2 (22.3) & 29.3 & 68.7 & 2.0 \\
1250.0 & 2000 & 0.00048 & 0.006 & 989.0 (21.9) & 28.8 & 69.3 & 1.9 \\
\bottomrule
\end{tabular}
\end{adjustbox}

\end{table}

\begin{table}
\centering
\small
\caption{ASI Monte Carlo errors and realized gold-label counts on Synthetic regression. Label counts report the retained-trial mean with sample standard deviation in parentheses.}
\label{tab:asi_diagnostics_regression}
\begin{adjustbox}{max width=\textwidth}
\begin{tabular}{lllllll}
\toprule
Budget & Cheap width SE & Exp. width SE & Cheap cov. SE & Exp. cov. SE & Cheap labels (SD) & Exp. labels (SD) \\
\midrule
75.0 & 0.0034 & -- & 0.0066 & -- & 34.9 (5.9) & -- \\
124.0 & 0.0079 & -- & 0.0066 & -- & 82.9 (8.9) & -- \\
172.9 & 0.0024 & -- & 0.0070 & -- & 131.9 (11.3) & -- \\
221.9 & 0.0007 & -- & 0.0071 & -- & 180.6 (12.6) & -- \\
270.8 & 0.00087 & -- & 0.0069 & -- & 230.1 (14.4) & -- \\
319.8 & 0.0023 & -- & 0.0070 & -- & 278.7 (15.8) & -- \\
368.8 & 0.002 & 0.031 & 0.0066 & 0.0069 & 328.1 (17.1) & 19.9 (4.5) \\
417.7 & 0.0014 & 0.014 & 0.0070 & 0.0065 & 377.4 (17.8) & 69.1 (8.3) \\
466.7 & 0.0017 & 0.0038 & 0.0065 & 0.0066 & 425.8 (18.9) & 118.4 (10.5) \\
515.6 & 0.0019 & 0.0055 & 0.0067 & 0.0067 & 475.3 (19.9) & 166.9 (12.6) \\
564.6 & 0.0013 & 0.0033 & 0.0068 & 0.0068 & 524.9 (20.9) & 215.7 (13.4) \\
613.5 & 0.0021 & 0.0047 & 0.0065 & 0.0067 & 572.4 (21.5) & 265.1 (14.7) \\
662.5 & 0.0025 & 0.0027 & 0.0069 & 0.0069 & 621.8 (22.4) & 314.1 (15.4) \\
711.5 & 0.00081 & 0.0096 & 0.0071 & 0.0068 & 671.4 (22.0) & 363.1 (16.4) \\
760.4 & 0.0011 & 0.0036 & 0.0067 & 0.0068 & 720.1 (22.9) & 411.9 (17.7) \\
809.4 & 0.00081 & 0.012 & 0.0067 & 0.0069 & 768.1 (22.7) & 460.7 (18.4) \\
858.3 & 0.0012 & 0.0021 & 0.0067 & 0.0069 & 818.1 (23.9) & 509.9 (17.7) \\
907.3 & 0.0012 & 0.0024 & 0.0066 & 0.0067 & 866.7 (23.8) & 559.0 (18.9) \\
956.2 & 0.0012 & 0.0063 & 0.0069 & 0.0067 & 915.7 (24.8) & 606.8 (18.7) \\
1005.2 & 0.0011 & 0.0055 & 0.0066 & 0.0066 & 963.8 (24.4) & 654.8 (19.3) \\
1054.2 & 0.00093 & 0.0052 & 0.0064 & 0.0065 & 1011.8 (24.4) & 702.0 (19.3) \\
1103.1 & 0.00095 & 0.0048 & 0.0067 & 0.0069 & 1060.7 (24.9) & 750.4 (19.5) \\
1152.1 & 0.00098 & 0.0045 & 0.0068 & 0.0066 & 1107.6 (25.7) & 796.4 (19.4) \\
1201.0 & 0.00099 & 0.0047 & 0.0068 & 0.0067 & 1156.8 (24.8) & 843.1 (19.7) \\
1250.0 & 0.00094 & 0.0044 & 0.0068 & 0.0065 & 1202.7 (25.6) & 887.5 (19.8) \\
\bottomrule
\end{tabular}
\end{adjustbox}

\end{table}

\begin{table}
\centering
\small
\caption{AM-PPI Monte Carlo precision, realized label counts, and calibration routing shares on MIMIC.}
\label{tab:diagnostics_mimic}
\begin{adjustbox}{max width=\textwidth}
\begin{tabular}{llllllll}
\toprule
Budget & Trials & Width SE & Coverage SE & Labels (SD) & Cheap only (\%) & Expensive only (\%) & Both (\%) \\
\midrule
125.0 & 1000 & 0.00092 & 0.009 & 48.0 (6.7) & 100.0 & 0.0 & 0.0 \\
132.4 & 1000 & 0.00083 & 0.009 & 54.8 (7.2) & 99.5 & 0.5 & 0.0 \\
139.7 & 1000 & 0.00076 & 0.010 & 59.8 (7.4) & 97.1 & 2.9 & 0.0 \\
147.1 & 1000 & 0.00068 & 0.010 & 61.9 (7.4) & 91.2 & 8.8 & 0.0 \\
154.4 & 1000 & 0.0007 & 0.009 & 62.5 (7.7) & 83.8 & 16.2 & 0.0 \\
161.8 & 1000 & 0.00068 & 0.010 & 64.0 (7.9) & 78.8 & 21.2 & 0.0 \\
169.1 & 1000 & 0.00065 & 0.009 & 66.0 (7.9) & 74.8 & 25.2 & 0.0 \\
176.5 & 1000 & 0.00062 & 0.009 & 70.4 (8.2) & 68.8 & 31.2 & 0.0 \\
183.8 & 1000 & 0.00056 & 0.010 & 75.6 (8.7) & 65.4 & 34.6 & 0.0 \\
191.2 & 1000 & 0.00051 & 0.010 & 81.1 (8.6) & 63.1 & 36.9 & 0.0 \\
198.5 & 1000 & 0.00048 & 0.009 & 84.1 (8.7) & 59.2 & 40.8 & 0.0 \\
205.9 & 1000 & 0.00047 & 0.010 & 85.6 (8.9) & 54.8 & 45.2 & 0.0 \\
213.2 & 1000 & 0.00051 & 0.010 & 85.6 (9.0) & 44.2 & 55.8 & 0.0 \\
220.6 & 1000 & 0.00044 & 0.010 & 88.7 (9.1) & 43.4 & 56.6 & 0.0 \\
227.9 & 1000 & 0.00056 & 0.010 & 92.0 (9.0) & 32.9 & 67.1 & 0.0 \\
235.3 & 1000 & 0.00054 & 0.010 & 95.9 (9.2) & 29.8 & 70.2 & 0.0 \\
242.6 & 1000 & 0.00047 & 0.010 & 102.3 (9.5) & 27.2 & 72.8 & 0.0 \\
250.0 & 1000 & 0.00045 & 0.009 & 105.9 (9.5) & 26.0 & 73.5 & 0.5 \\
257.4 & 1000 & 0.00032 & 0.010 & 109.8 (9.7) & 24.6 & 75.0 & 0.4 \\
264.7 & 1000 & 0.00033 & 0.010 & 113.7 (10.0) & 21.0 & 78.0 & 0.9 \\
272.1 & 1000 & 0.00029 & 0.010 & 121.0 (9.9) & 20.1 & 79.9 & 0.0 \\
279.4 & 1000 & 0.00028 & 0.010 & 127.7 (10.9) & 19.6 & 80.4 & 0.0 \\
286.8 & 1000 & 0.00027 & 0.009 & 129.8 (10.4) & 14.1 & 84.2 & 1.7 \\
294.1 & 1000 & 0.00032 & 0.010 & 139.7 (10.5) & 12.3 & 87.6 & 0.1 \\
301.5 & 1000 & 0.00024 & 0.010 & 141.7 (10.8) & 10.3 & 87.1 & 2.6 \\
308.8 & 1000 & 0.00025 & 0.010 & 147.4 (10.9) & 9.5 & 87.6 & 2.9 \\
316.2 & 1000 & 0.00022 & 0.009 & 153.5 (11.3) & 8.5 & 87.1 & 4.4 \\
323.5 & 1000 & 0.00019 & 0.010 & 162.5 (11.1) & 8.4 & 91.6 & 0.0 \\
330.9 & 1000 & 0.00021 & 0.010 & 167.2 (11.0) & 7.8 & 87.3 & 4.8 \\
338.2 & 1000 & 0.00017 & 0.010 & 175.3 (11.7) & 5.5 & 94.5 & 0.0 \\
345.6 & 1000 & 0.00022 & 0.010 & 188.0 (11.5) & 9.7 & 90.3 & 0.0 \\
352.9 & 1000 & 0.0002 & 0.009 & 190.7 (11.6) & 10.3 & 89.5 & 0.1 \\
360.3 & 1000 & 0.00019 & 0.009 & 199.3 (11.9) & 10.3 & 89.5 & 0.1 \\
367.6 & 1000 & 0.00019 & 0.009 & 207.8 (12.1) & 11.0 & 89.0 & 0.0 \\
375.0 & 1000 & 0.00018 & 0.010 & 216.0 (12.1) & 10.8 & 89.0 & 0.1 \\
\bottomrule
\end{tabular}
\end{adjustbox}

\end{table}

\begin{table}
\centering
\small
\caption{ASI Monte Carlo errors and realized gold-label counts on MIMIC. }
\label{tab:asi_diagnostics_mimic}
\begin{adjustbox}{max width=\textwidth}
\begin{tabular}{lllllll}
\toprule
Budget & Cheap width SE & Exp. width SE & Cheap cov. SE & Exp. cov. SE & Cheap labels (SD) & Exp. labels (SD) \\
\midrule
125.0 & 0.001 & -- & 0.0097 & -- & 48.4 (6.7) & -- \\
132.4 & 0.00089 & -- & 0.0089 & -- & 54.9 (7.0) & -- \\
139.7 & 0.00077 & -- & 0.0096 & -- & 62.7 (7.5) & -- \\
147.1 & 0.00069 & -- & 0.0094 & -- & 69.6 (7.9) & -- \\
154.4 & 0.00062 & -- & 0.0097 & -- & 76.6 (8.5) & -- \\
161.8 & 0.00056 & -- & 0.0096 & -- & 85.0 (8.6) & -- \\
169.1 & 0.0005 & -- & 0.0094 & -- & 91.8 (8.7) & -- \\
176.5 & 0.00045 & -- & 0.0092 & -- & 98.7 (9.0) & -- \\
183.8 & 0.00043 & 0.0098 & 0.0102 & 0.0293 & 106.8 (9.9) & 21.7 (2.0) \\
191.2 & 0.0004 & 0.0029 & 0.0097 & 0.0120 & 114.4 (9.9) & 24.2 (3.2) \\
198.5 & 0.00036 & 0.0017 & 0.0097 & 0.0111 & 121.6 (10.4) & 30.1 (5.0) \\
205.9 & 0.00034 & 0.0014 & 0.0101 & 0.0111 & 129.0 (10.5) & 36.9 (5.8) \\
213.2 & 0.00032 & 0.0011 & 0.0100 & 0.0112 & 136.1 (10.5) & 44.3 (6.6) \\
220.6 & 0.0003 & 0.00094 & 0.0094 & 0.0106 & 143.1 (10.8) & 51.9 (6.8) \\
227.9 & 0.00027 & 0.00079 & 0.0090 & 0.0106 & 150.3 (10.9) & 59.1 (7.5) \\
235.3 & 0.00026 & 0.00065 & 0.0091 & 0.0093 & 157.5 (11.0) & 66.3 (7.9) \\
242.6 & 0.00025 & 0.0006 & 0.0094 & 0.0097 & 165.9 (11.3) & 74.2 (7.8) \\
250.0 & 0.00023 & 0.00057 & 0.0094 & 0.0108 & 172.9 (11.8) & 81.1 (8.7) \\
257.4 & 0.00022 & 0.00049 & 0.0093 & 0.0097 & 180.7 (12.1) & 88.0 (8.8) \\
264.7 & 0.00021 & 0.00044 & 0.0092 & 0.0098 & 188.0 (12.0) & 96.6 (9.3) \\
272.1 & 0.0002 & 0.00041 & 0.0101 & 0.0101 & 195.1 (12.0) & 102.5 (9.7) \\
279.4 & 0.00019 & 0.00037 & 0.0097 & 0.0097 & 202.1 (12.1) & 109.8 (9.8) \\
286.8 & 0.00018 & 0.00034 & 0.0094 & 0.0094 & 210.0 (12.4) & 118.0 (9.9) \\
294.1 & 0.00017 & 0.00031 & 0.0096 & 0.0096 & 216.9 (12.8) & 125.4 (10.2) \\
301.5 & 0.00016 & 0.00032 & 0.0093 & 0.0097 & 224.0 (12.7) & 131.6 (10.8) \\
308.8 & 0.00015 & 0.00028 & 0.0090 & 0.0101 & 231.9 (12.3) & 140.2 (10.8) \\
316.2 & 0.00014 & 0.00025 & 0.0096 & 0.0093 & 239.1 (12.3) & 146.2 (11.0) \\
323.5 & 0.00015 & 0.00024 & 0.0095 & 0.0101 & 246.7 (13.3) & 154.9 (11.0) \\
330.9 & 0.00013 & 0.00023 & 0.0095 & 0.0097 & 253.4 (12.9) & 161.5 (11.2) \\
338.2 & 0.00014 & 0.00021 & 0.0099 & 0.0091 & 260.5 (13.1) & 169.2 (11.8) \\
345.6 & 0.00013 & 0.00019 & 0.0098 & 0.0097 & 267.8 (13.1) & 176.8 (11.1) \\
352.9 & 0.00012 & 0.00018 & 0.0098 & 0.0099 & 275.2 (13.4) & 183.5 (12.0) \\
360.3 & 0.00012 & 0.00017 & 0.0099 & 0.0097 & 283.2 (12.9) & 190.9 (12.0) \\
367.6 & 0.00011 & 0.00017 & 0.0095 & 0.0098 & 290.3 (13.7) & 198.7 (11.6) \\
375.0 & 0.0001 & 0.00016 & 0.0088 & 0.0096 & 298.3 (13.0) & 205.6 (12.3) \\
\bottomrule
\end{tabular}
\end{adjustbox}

\end{table}

\begin{table}
\centering
\small
\caption{AM-PPI Monte Carlo precision, realized label counts, and calibration routing shares on Adult income.}
\label{tab:diagnostics_adult_income}
\begin{adjustbox}{max width=\textwidth}
\begin{tabular}{llllllll}
\toprule
Budget & Trials & Width SE & Coverage SE & Labels (SD) & Cheap only (\%) & Expensive only (\%) & Both (\%) \\
\midrule
72.0 & 1000 & 0.0005 & 0.010 & 66.1 (7.9) & 76.5 & 15.0 & 8.5 \\
84.5 & 1000 & 0.00044 & 0.010 & 78.4 (8.8) & 72.6 & 15.4 & 12.0 \\
99.6 & 1000 & 0.00038 & 0.010 & 92.6 (9.2) & 71.1 & 15.5 & 13.5 \\
117.8 & 1000 & 0.00032 & 0.009 & 111.0 (10.3) & 67.3 & 15.8 & 16.9 \\
139.8 & 1000 & 0.00026 & 0.009 & 133.0 (11.5) & 61.4 & 15.8 & 22.8 \\
166.4 & 1000 & 0.00022 & 0.009 & 157.7 (11.9) & 54.8 & 17.3 & 28.0 \\
198.5 & 1000 & 0.00018 & 0.009 & 190.2 (13.0) & 53.7 & 17.0 & 29.3 \\
237.3 & 1000 & 0.00016 & 0.009 & 229.6 (15.1) & 51.6 & 16.2 & 32.2 \\
284.2 & 1000 & 0.00013 & 0.010 & 275.1 (15.5) & 48.8 & 16.5 & 34.8 \\
340.9 & 1000 & 0.00011 & 0.010 & 331.7 (17.5) & 46.2 & 17.3 & 36.5 \\
409.3 & 1000 & 8.8e-05 & 0.010 & 400.3 (17.8) & 44.5 & 17.2 & 38.2 \\
492.0 & 1000 & 7.2e-05 & 0.009 & 482.8 (19.7) & 46.2 & 17.4 & 36.4 \\
\bottomrule
\end{tabular}
\end{adjustbox}

\end{table}

\begin{table}
\centering
\small
\caption{ASI Monte Carlo errors and realized gold-label counts on Adult income.}
\label{tab:asi_diagnostics_adult_income}
\begin{adjustbox}{max width=\textwidth}
\begin{tabular}{lllllll}
\toprule
Budget & Cheap width SE & Exp. width SE & Cheap cov. SE & Exp. cov. SE & Cheap labels (SD) & Exp. labels (SD) \\
\midrule
72.0 & 0.00065 & 0.00071 & 0.0095 & 0.0105 & 69.2 (8.1) & 60.0 (7.8) \\
84.5 & 0.00053 & 0.0006 & 0.0101 & 0.0105 & 81.9 (8.7) & 72.5 (8.1) \\
99.6 & 0.00044 & 0.0005 & 0.0090 & 0.0100 & 97.1 (9.7) & 87.0 (8.8) \\
117.8 & 0.00038 & 0.00041 & 0.0092 & 0.0097 & 114.8 (10.6) & 105.2 (10.1) \\
139.8 & 0.0003 & 0.00032 & 0.0099 & 0.0102 & 136.5 (10.9) & 126.3 (10.5) \\
166.4 & 0.00026 & 0.00028 & 0.0094 & 0.0095 & 163.1 (12.3) & 154.1 (11.7) \\
198.5 & 0.00022 & 0.00022 & 0.0098 & 0.0102 & 195.7 (12.7) & 186.0 (13.3) \\
237.3 & 0.00018 & 0.00018 & 0.0093 & 0.0090 & 234.5 (15.2) & 224.9 (14.3) \\
284.2 & 0.00015 & 0.00015 & 0.0096 & 0.0093 & 280.7 (16.0) & 272.1 (15.1) \\
340.9 & 0.00013 & 0.00012 & 0.0090 & 0.0095 & 338.2 (16.3) & 328.0 (16.4) \\
409.3 & 0.00011 & 9.6e-05 & 0.0097 & 0.0092 & 406.6 (18.3) & 396.8 (17.3) \\
492.0 & 8.2e-05 & 8.4e-05 & 0.0099 & 0.0098 & 488.0 (18.9) & 479.7 (19.0) \\
\bottomrule
\end{tabular}
\end{adjustbox}

\end{table}

\begin{table}
\centering
\small
\caption{AM-PPI Monte Carlo precision, realized label counts, and calibration routing shares on Bank Marketing.}
\label{tab:diagnostics_bank_marketing}
\begin{adjustbox}{max width=\textwidth}
\begin{tabular}{llllllll}
\toprule
Budget & Trials & Width SE & Coverage SE & Labels (SD) & Cheap only (\%) & Expensive only (\%) & Both (\%) \\
\midrule
823.8 & 1000 & 6.2e-05 & 0.010 & 700.2 (25.6) & 94.4 & 5.6 & 0.0 \\
901.1 & 1000 & 5.3e-05 & 0.009 & 774.2 (26.8) & 93.5 & 6.5 & 0.0 \\
997.7 & 1000 & 4.8e-05 & 0.009 & 866.3 (28.5) & 92.6 & 7.4 & 0.0 \\
1118.5 & 1000 & 4.3e-05 & 0.010 & 984.2 (30.5) & 92.2 & 7.8 & 0.0 \\
1269.6 & 1000 & 3.7e-05 & 0.009 & 1136.6 (32.8) & 91.6 & 8.3 & 0.0 \\
1458.5 & 1000 & 3.1e-05 & 0.009 & 1319.3 (34.1) & 91.1 & 8.8 & 0.0 \\
1694.6 & 1000 & 2.7e-05 & 0.009 & 1555.9 (36.2) & 90.6 & 9.4 & 0.1 \\
1989.9 & 1000 & 2.3e-05 & 0.009 & 1848.7 (39.2) & 90.5 & 9.4 & 0.1 \\
2359.0 & 1000 & 1.9e-05 & 0.009 & 2217.5 (42.0) & 90.2 & 9.6 & 0.1 \\
2820.5 & 1000 & 1.6e-05 & 0.010 & 2677.7 (43.1) & 90.0 & 9.8 & 0.2 \\
3397.4 & 1000 & 1.3e-05 & 0.010 & 3258.9 (46.0) & 90.0 & 9.9 & 0.1 \\
4118.8 & 1000 & 1e-05 & 0.009 & 3977.4 (49.0) & 90.0 & 9.9 & 0.1 \\
\bottomrule
\end{tabular}
\end{adjustbox}

\end{table}

\begin{table}
\centering
\small
\caption{ASI Monte Carlo errors and realized gold-label counts on Bank Marketing. }
\label{tab:asi_diagnostics_bank_marketing}
\begin{adjustbox}{max width=\textwidth}
\begin{tabular}{lllllll}
\toprule
Budget & Cheap width SE & Exp. width SE & Cheap cov. SE & Exp. cov. SE & Cheap labels (SD) & Exp. labels (SD) \\
\midrule
823.8 & 6.2e-05 & 0.00015 & 0.0095 & 0.0100 & 720.8 (25.8) & 308.5 (17.6) \\
901.1 & 5.5e-05 & 0.00012 & 0.0096 & 0.0097 & 797.5 (27.3) & 387.0 (19.3) \\
997.7 & 4.9e-05 & 9.9e-05 & 0.0092 & 0.0101 & 893.6 (29.6) & 482.4 (21.3) \\
1118.5 & 4.5e-05 & 8.3e-05 & 0.0101 & 0.0102 & 1013.1 (30.1) & 601.7 (23.9) \\
1269.6 & 3.7e-05 & 6.3e-05 & 0.0089 & 0.0092 & 1166.5 (30.9) & 752.6 (27.2) \\
1458.5 & 3.2e-05 & 4.8e-05 & 0.0096 & 0.0089 & 1354.4 (32.9) & 943.4 (29.3) \\
1694.6 & 2.9e-05 & 3.8e-05 & 0.0102 & 0.0099 & 1591.3 (37.4) & 1181.0 (31.5) \\
1989.9 & 2.4e-05 & 3.1e-05 & 0.0094 & 0.0092 & 1884.8 (40.2) & 1477.2 (36.2) \\
2359.0 & 2.1e-05 & 2.5e-05 & 0.0099 & 0.0097 & 2257.1 (41.8) & 1844.5 (37.0) \\
2820.5 & 1.7e-05 & 1.9e-05 & 0.0099 & 0.0090 & 2717.5 (43.6) & 2305.7 (41.2) \\
3397.4 & 1.2e-05 & 1.6e-05 & 0.0086 & 0.0096 & 3295.2 (48.6) & 2883.6 (44.4) \\
4118.8 & 1.1e-05 & 1.3e-05 & 0.0099 & 0.0101 & 4015.8 (47.5) & 3603.0 (46.4) \\
\bottomrule
\end{tabular}
\end{adjustbox}

\end{table}

\paragraph{All method comparison details.} 

The overall comparison is reported in Table~\ref{tab:grand_overall} in the main text and the dataset-specific ranks are given below.The randomized comparison draws 3 cost configurations per dataset and 4 candidate budgets per configuration, with 25 paired trials per distinct budget. Cheap cost is its main-study value times a log-uniform factor in [0.5,2]; the expensive-to-cheap ratio is log-uniform in [1,4]. The label allowance is a rounded log-uniform draw between the larger of three times the reporting label minimum and 3 percent of the test sample, and 30 percent of the test sample. Rank 1 denotes the narrowest interval; widths within 2.5 percent of the narrowest member of a tie group share its competition rank. Coverage SD is the standard deviation of scenario-level empirical coverages.  Average percent improvement versus classical is the arithmetic mean of 100*(classical width - method width)/classical width over matched scenario-trial cells and positive values indicate narrower intervals. 

\begin{table}
\centering
\caption{Average CI-width ranks by dataset (1 = narrowest, ties within 2.5\%).}
\label{tab:grand_dataset_ranks}
\begin{adjustbox}{max width=\textwidth}
\begin{tabular}{lrrrr}
\toprule
Method & Adult income & Bank Marketing & MIMIC & Regression \\
\midrule
AM-PPI (intercept, ours) & 2.70 & 3.48 & 3.54 & 3.27 \\
AM-PPI (ours) & 2.90 & 3.13 & 4.85 & 2.83 \\
AM-PPI-1 (expensive, ours) & 5.18 & 4.51 & 6.94 & 2.62 \\
Cascade (ours) & 3.58 & 4.51 & 8.11 & 4.80 \\
ASI (expensive) & 6.61 & 4.39 & 7.60 & 4.08 \\
AM-PPI-1 (expensive, intercept, ours) & 8.90 & 4.70 & 6.89 & 2.44 \\
AM-PPI-1 (cheap, intercept, ours) & 4.67 & 5.30 & 6.76 & 8.66 \\
AM-PPI-1 (cheap, ours) & 5.49 & 5.56 & 9.89 & 8.61 \\
Stratified PPI (expensive) & 11.92 & 8.16 & 5.33 & 4.63 \\
PPI++ (expensive) & 9.20 & 8.01 & 6.17 & 9.33 \\
PPI (expensive) & 9.62 & 8.54 & 8.94 & 10.40 \\
ASI (cheap) & 6.59 & 10.39 & 11.12 & 9.98 \\
Stratified PPI (cheap) & 11.82 & 11.84 & 6.13 & 12.04 \\
PPI++ (cheap) & 10.15 & 11.57 & 6.78 & 13.51 \\
PPI (cheap) & 10.18 & 11.02 & 13.75 & 13.87 \\
\bottomrule
\end{tabular}
\end{adjustbox}

\end{table}

\clearpage
\subsection{Predictor-cost sensitivity}
\label{sec:rebuttal-costs}
The cost sweeps reuse each study's predictors and data splits while varying the expensive-to-cheap cost ratio or scaling both costs together. 

\begin{table}
\centering
\caption{AM-PPI interval width as a percent of the better true-ASI singleton on MIMIC, by expensive predictor cost $c_2$ and total budget $B$.}
\label{tab:sensitivity_mimic_ratio}
\begin{adjustbox}{max width=\textwidth}
\begin{tabular}{lrrrrrrr}
\toprule
$c_2$ / $B$ & 160 & 185 & 220 & 250 & 280 & 320 & 375 \\
\midrule
0.1 & 100.3 & 94.3 & 100.7 & 97.8 & 101.0 & 101.2 & 97.6 \\
0.15 & 93.5 & 98.3 & 95.1 & 100.6 & 95.9 & 101.2 & 97.4 \\
0.22 & 93.8 & 89.7 & 83.3 & 91.8 & 92.4 & 96.5 & 100.2 \\
0.28 & -- & 95.9 & 93.5 & 91.9 & 87.2 & 86.7 & 93.1 \\
0.35 & -- & -- & 95.2 & 97.1 & 95.1 & 92.6 & 90.4 \\
\bottomrule
\end{tabular}
\end{adjustbox}

\end{table}
\begin{table}
\centering
\caption{AM-PPI interval width as a percent of the better true-ASI singleton on Synthetic regression, by expensive predictor cost $c_2$ and total budget $B$.}
\label{tab:sensitivity_regression_ratio}
\begin{adjustbox}{max width=\textwidth}
\begin{tabular}{lrrrrrrr}
\toprule
$c_2$ / $B$ & 75 & 150 & 250 & 400 & 650 & 900 & 1250 \\
\midrule
0.01 & 92.3 & 92.2 & 93.7 & 96.1 & 84.5 & 93.4 & 87.6 \\
0.025 & -- & 78.6 & 89.8 & 94.4 & 84.6 & 93.0 & 87.6 \\
0.05 & -- & 87.5 & 82.3 & 87.9 & 81.8 & 93.4 & 86.3 \\
0.087 & -- & -- & 92.0 & 87.0 & 83.4 & 90.2 & 86.9 \\
0.1 & -- & -- & 94.2 & 89.0 & 84.6 & 89.8 & 84.8 \\
\bottomrule
\end{tabular}
\end{adjustbox}

\end{table}

\begin{table}
\centering
\caption{AM-PPI interval width as a percent of the better true-ASI singleton on Adult, by expensive predictor cost $c_2$ and total budget $B$.}
\label{tab:sensitivity_adult_income_ratio}
\begin{adjustbox}{max width=\textwidth}
\begin{tabular}{lrrrrrrr}
\toprule
$c_2$ / $B$ & 75 & 100 & 150 & 250 & 350 & 500 & 600 \\
\midrule
0.001 & 89.6 & 89.1 & 89.6 & 90.5 & 91.5 & 91.7 & 92.2 \\
0.002 & 89.2 & 89.1 & 89.4 & 90.2 & 91.4 & 91.5 & 92.1 \\
0.0035 & 88.7 & 89.9 & 90.0 & 89.8 & 91.1 & 91.5 & 92.0 \\
0.005 & 89.4 & 89.8 & 90.5 & 90.0 & 91.2 & 91.7 & 92.1 \\
0.01 & 91.0 & 91.2 & 91.2 & 90.5 & 91.9 & 92.2 & 92.5 \\
\bottomrule
\end{tabular}
\end{adjustbox}

\end{table}

\begin{table}
\centering
\caption{AM-PPI interval width as a percent of the better true-ASI singleton on Bank, by expensive predictor cost $c_2$ and total budget $B$.}
\label{tab:sensitivity_bank_marketing_ratio}
\begin{adjustbox}{max width=\textwidth}
\begin{tabular}{lrrrrrrr}
\toprule
$c_2$ / $B$ & 400 & 600 & 850 & 1200 & 1800 & 2600 & 3500 \\
\midrule
0.01 & 100.8 & 100.9 & 100.5 & 101.3 & 100.9 & 100.6 & 100.8 \\
0.02 & 90.8 & 92.6 & 95.8 & 99.4 & 100.1 & 100.2 & 100.5 \\
0.035 & 91.8 & 91.9 & 91.2 & 91.4 & 95.4 & 98.2 & 99.2 \\
0.05 & 92.6 & 92.5 & 91.5 & 91.6 & 91.9 & 94.4 & 96.5 \\
0.1 & -- & 93.8 & 92.8 & 92.4 & 92.5 & 92.4 & 92.9 \\
\bottomrule
\end{tabular}
\end{adjustbox}

\end{table}

\clearpage
\subsection{Calibration diagnostics}
\label{sec:rebuttal-diagnostics}

\begin{table}
\centering
\small
\caption{Calibration stationarity and budget diagnostics across synthetic parameter sweeps. Stationarity means relative weight and multiplier drift below 0.05 on each of two further refinement sweeps. Diagnostics describe numerical behavior and do not imply an optimization guarantee. }
\label{tab:calibration_diagnostics}
\begin{adjustbox}{max width=\textwidth}
\begin{tabular}{llllll}
\toprule
Sweep & Value & Runs & Stationary & Max relative budget gap & Max slackness product \\
\midrule
budget\_per\_instance & 0.12 & 5 & 5 & 2.54e-13 & 2.59e-15 \\
budget\_per\_instance & 0.2 & 5 & 5 & 2.71e-13 & 1.46e-15 \\
budget\_per\_instance & 0.35 & 5 & 5 & 2.22e-13 & 6.05e-16 \\
budget\_per\_instance & 0.6 & 5 & 5 & 1.36e-13 & 2.00e-16 \\
budget\_per\_instance & 1.0 & 5 & 3 & 3.68e-03 & 1.33e-06 \\
budget\_per\_instance & 1.6 & 5 & 5 & 3.69e-01 & 3.57e-26 \\
t\_max & 10 & 5 & 4 & 4.03e-13 & 1.25e-15 \\
n & 150 & 5 & 2 & 2.47e-13 & 2.17e-15 \\
n & 400 & 5 & 5 & 2.58e-13 & 7.23e-16 \\
n & 800 & 5 & 5 & 1.56e-13 & 2.47e-16 \\
cost\_ratio & 1.0 & 5 & 5 & 3.26e-13 & 6.51e-16 \\
cost\_ratio & 3.0 & 5 & 5 & 3.48e-13 & 8.96e-16 \\
cost\_ratio & 8.0 & 5 & 5 & 2.25e-13 & 6.45e-16 \\
cost\_ratio & 20.0 & 5 & 4 & 2.96e-13 & 7.64e-16 \\
noise\_scale & 0.3 & 5 & 5 & 2.72e-13 & 1.33e-16 \\
noise\_scale & 1.0 & 5 & 5 & 1.89e-13 & 5.10e-16 \\
noise\_scale & 2.5 & 5 & 5 & 6.07e-07 & 8.92e-09 \\
pred\_err & 0.25 & 5 & 3 & 3.75e-13 & 6.35e-16 \\
pred\_err & 1.0 & 5 & 4 & 3.60e-13 & 1.11e-15 \\
pred\_err & 2.0 & 5 & 5 & 3.10e-13 & 1.09e-15 \\
\bottomrule
\end{tabular}
\end{adjustbox}

\end{table}

\clearpage
\subsection{Calibration shift and covariate transfer}
\label{sec:rebuttal-shift}
The transfer stress test assigns records to a deployment stratum with probability $\operatorname{sigmoid}(\beta z)$, where $z$ is standardized $\log(1+\mathrm{note\_tokens})$ on MIMIC and $\log(1+\mathrm{age})$ on Adult and Bank.  Control calibration uses the target stratum, stale calibration uses the source stratum, and recalibration uses a quarter-size fresh target sample. For unequal calibration and test sizes the deployment multiplier is rescaled as $\mu_{\mathrm{test}}=\mu_{\mathrm{cal}}n_{\mathrm{cal}}/n_{\mathrm{test}}$, preserving the routing scale $n\mu$. 
\begin{table}
\centering
\small
\caption{Covariate-shift stress test on MIMIC; nominal coverage 0.90. $\beta$: shift strength; Arm: calibration regime; Method: estimator; Trials: pooled trial count; Width (Width SE): mean CI width (its Monte Carlo SE); Coverage (Cov. SE): empirical coverage (its Monte Carlo SE); Labels: mean queried labels.}
\label{tab:drift_mimic}
\begin{adjustbox}{max width=\textwidth}
\begin{tabular}{lllllllll}
\toprule
beta & Arm & Method & Trials & Width & Width SE & Coverage & Cov. SE & Labels \\
\midrule
0.5 & control & AM-PPI & 100 & 0.1250 & 0.0025 & 0.890 & 0.031 & 87.6 \\
0.5 & control & Cascade (ours) & 100 & 0.1452 & 0.0024 & 0.910 & 0.029 & 109.2 \\
0.5 & recal & AM-PPI & 100 & 0.1453 & 0.0044 & 0.890 & 0.031 & 85.0 \\
0.5 & recal & Cascade (ours) & 100 & 0.1743 & 0.0046 & 0.870 & 0.034 & 66.0 \\
0.5 & stale & AM-PPI & 100 & 0.1295 & 0.0025 & 0.920 & 0.027 & 86.9 \\
0.5 & stale & Cascade (ours) & 100 & 0.1598 & 0.0036 & 0.920 & 0.027 & 90.5 \\
1.0 & control & AM-PPI & 100 & 0.1264 & 0.0027 & 0.910 & 0.029 & 92.8 \\
1.0 & control & Cascade (ours) & 100 & 0.1512 & 0.0024 & 0.900 & 0.030 & 113.0 \\
1.0 & recal & AM-PPI & 100 & 0.1372 & 0.0042 & 0.870 & 0.034 & 75.5 \\
1.0 & recal & Cascade (ours) & 100 & 0.1684 & 0.0026 & 0.880 & 0.032 & 77.8 \\
1.0 & stale & AM-PPI & 100 & 0.1315 & 0.0029 & 0.910 & 0.029 & 89.7 \\
1.0 & stale & Cascade (ours) & 100 & 0.1576 & 0.0032 & 0.890 & 0.031 & 105.4 \\
2.0 & control & AM-PPI & 100 & 0.1339 & 0.0034 & 0.890 & 0.031 & 82.7 \\
2.0 & control & Cascade (ours) & 100 & 0.1539 & 0.0027 & 0.910 & 0.029 & 101.0 \\
2.0 & recal & AM-PPI & 100 & 0.1404 & 0.0035 & 0.920 & 0.027 & 92.7 \\
2.0 & recal & Cascade (ours) & 100 & 0.1584 & 0.0031 & 0.880 & 0.032 & 93.7 \\
2.0 & stale & AM-PPI & 100 & 0.1469 & 0.0049 & 0.880 & 0.032 & 90.2 \\
2.0 & stale & Cascade (ours) & 100 & 0.1680 & 0.0036 & 0.920 & 0.027 & 105.5 \\
\bottomrule
\end{tabular}
\end{adjustbox}

\end{table}

\begin{table}
\centering
\small
\caption{Covariate-shift stress test on Adult income; nominal coverage 0.90. $\beta$: shift strength; Arm: calibration regime; Method: estimator; Trials: pooled trial count; Width (Width SE): mean CI width (its Monte Carlo SE); Coverage (Cov. SE): empirical coverage (its Monte Carlo SE); Labels: mean queried labels.}
\label{tab:drift_adult_income}
\begin{adjustbox}{max width=\textwidth}
\begin{tabular}{lllllllll}
\toprule
beta & Arm & Method & Trials & Width & Width SE & Coverage & Cov. SE & Labels \\
\midrule
0.5 & control & AM-PPI & 100 & 0.1446 & 0.004 & 0.890 & 0.031 & 71.5 \\
0.5 & control & Cascade (ours) & 100 & 0.1489 & 0.0043 & 0.910 & 0.029 & 73.3 \\
0.5 & recal & AM-PPI & 100 & 0.1384 & 0.0044 & 0.890 & 0.031 & 69.6 \\
0.5 & recal & Cascade (ours) & 100 & 0.1406 & 0.004 & 0.930 & 0.026 & 70.7 \\
0.5 & stale & AM-PPI & 100 & 0.1434 & 0.0043 & 0.860 & 0.035 & 70.2 \\
0.5 & stale & Cascade (ours) & 100 & 0.1455 & 0.0044 & 0.900 & 0.030 & 72.1 \\
1.0 & control & AM-PPI & 100 & 0.1421 & 0.0046 & 0.920 & 0.027 & 69.9 \\
1.0 & control & Cascade (ours) & 100 & 0.1453 & 0.0046 & 0.900 & 0.030 & 70.9 \\
1.0 & recal & AM-PPI & 100 & 0.1442 & 0.0043 & 0.930 & 0.026 & 70.0 \\
1.0 & recal & Cascade (ours) & 100 & 0.1497 & 0.0044 & 0.870 & 0.034 & 72.1 \\
1.0 & stale & AM-PPI & 100 & 0.1415 & 0.0042 & 0.920 & 0.027 & 69.9 \\
1.0 & stale & Cascade (ours) & 100 & 0.1433 & 0.0039 & 0.900 & 0.030 & 72.2 \\
2.0 & control & AM-PPI & 100 & 0.1700 & 0.0052 & 0.870 & 0.034 & 72.4 \\
2.0 & control & Cascade (ours) & 100 & 0.1696 & 0.0049 & 0.850 & 0.036 & 72.5 \\
2.0 & recal & AM-PPI & 100 & 0.1565 & 0.0043 & 0.830 & 0.038 & 70.5 \\
2.0 & recal & Cascade (ours) & 100 & 0.1632 & 0.0045 & 0.870 & 0.034 & 70.9 \\
2.0 & stale & AM-PPI & 100 & 0.1691 & 0.0049 & 0.850 & 0.036 & 71.2 \\
2.0 & stale & Cascade (ours) & 100 & 0.1764 & 0.0056 & 0.870 & 0.034 & 73.1 \\
\bottomrule
\end{tabular}
\end{adjustbox}

\end{table}

\begin{table}
\centering
\small
\caption{Covariate-shift stress test on Bank Marketing; nominal coverage 0.90. $\beta$: shift strength; Arm: calibration regime; Method: estimator; Trials: pooled trial count; Width (Width SE): mean CI width (its Monte Carlo SE); Coverage (Cov. SE): empirical coverage (its Monte Carlo SE); Labels: mean queried labels.}
\label{tab:drift_bank_marketing}
\begin{adjustbox}{max width=\textwidth}
\begin{tabular}{lllllllll}
\toprule
beta & Arm & Method & Trials & Width & Width SE & Coverage & Cov. SE & Labels \\
\midrule
0.5 & control & AM-PPI & 100 & 0.0959 & 0.0029 & 0.820 & 0.038 & 94.1 \\
0.5 & control & Cascade (ours) & 100 & 0.0947 & 0.0028 & 0.890 & 0.031 & 97.7 \\
0.5 & recal & AM-PPI & 100 & 0.0989 & 0.0029 & 0.850 & 0.036 & 95.2 \\
0.5 & recal & Cascade (ours) & 100 & 0.0976 & 0.0028 & 0.890 & 0.031 & 99.3 \\
0.5 & stale & AM-PPI & 100 & 0.0963 & 0.0025 & 0.860 & 0.035 & 91.9 \\
0.5 & stale & Cascade (ours) & 100 & 0.0945 & 0.0025 & 0.900 & 0.030 & 97.2 \\
1.0 & control & AM-PPI & 100 & 0.0946 & 0.0026 & 0.890 & 0.031 & 92.3 \\
1.0 & control & Cascade (ours) & 100 & 0.0932 & 0.0026 & 0.840 & 0.037 & 97.5 \\
1.0 & recal & AM-PPI & 100 & 0.1094 & 0.0034 & 0.930 & 0.026 & 95.0 \\
1.0 & recal & Cascade (ours) & 100 & 0.1044 & 0.0034 & 0.830 & 0.038 & 96.8 \\
1.0 & stale & AM-PPI & 100 & 0.0986 & 0.0027 & 0.860 & 0.035 & 90.4 \\
1.0 & stale & Cascade (ours) & 100 & 0.0940 & 0.0028 & 0.880 & 0.032 & 97.8 \\
2.0 & control & AM-PPI & 100 & 0.0826 & 0.0025 & 0.880 & 0.032 & 91.3 \\
2.0 & control & Cascade (ours) & 100 & 0.0762 & 0.0019 & 0.850 & 0.036 & 97.8 \\
2.0 & recal & AM-PPI & 100 & 0.0866 & 0.0034 & 0.890 & 0.031 & 97.6 \\
2.0 & recal & Cascade (ours) & 100 & 0.0833 & 0.005 & 0.820 & 0.038 & 99.2 \\
2.0 & stale & AM-PPI & 100 & 0.0848 & 0.0021 & 0.890 & 0.031 & 92.5 \\
2.0 & stale & Cascade (ours) & 100 & 0.0803 & 0.0019 & 0.890 & 0.031 & 97.8 \\
\bottomrule
\end{tabular}
\end{adjustbox}

\end{table}

\begin{table}
\centering
\small
\caption{Shift diagnostics and calibration label counts. Setting: dataset and $\beta$; Shift AUC: cross-validated AUC for source versus target classification; Residual correlation: target Spearman correlation of predicted and observed absolute cheap-predictor residuals, using a source-fitted residual model; Control/Fresh labels: full/quarter-size target calibration counts; Source/Target mean: respective outcome means.}
\label{tab:drift_diagnostics}
\begin{adjustbox}{max width=\textwidth}
\begin{tabular}{lllllll}
\toprule
Setting & Shift AUC & Residual correlation & Control labels & Fresh labels & Source mean & Target mean \\
\midrule
adult\_income\_beta0.5 & 0.609 & 0.831 & 500 & 125 & 0.210 & 0.265 \\
adult\_income\_beta1.0 & 0.715 & 0.822 & 500 & 125 & 0.186 & 0.289 \\
adult\_income\_beta2.0 & 0.862 & 0.715 & 500 & 125 & 0.154 & 0.319 \\
bank\_marketing\_beta0.5 & 0.637 & 0.595 & 500 & 125 & 0.111 & 0.105 \\
bank\_marketing\_beta1.0 & 0.690 & 0.411 & 500 & 125 & 0.114 & 0.102 \\
bank\_marketing\_beta2.0 & 0.846 & 0.524 & 500 & 125 & 0.116 & 0.099 \\
mimic\_beta0.5 & 0.594 & 0.124 & 500 & 125 & 0.492 & 0.509 \\
mimic\_beta1.0 & 0.710 & 0.100 & 500 & 125 & 0.494 & 0.507 \\
mimic\_beta2.0 & 0.844 & 0.030 & 500 & 125 & 0.495 & 0.505 \\
\bottomrule
\end{tabular}
\end{adjustbox}

\end{table}

\begin{table}
\centering
\small
\caption{Reconstructed synthetic calibration-shift experiment (budget 600, nominal coverage 0.90). Predictors and the inference population stay fixed. The outcome-driving calibration coordinate is shifted, outcomes are redrawn from the original conditional law, and calibration quantities are refit. Shift: additive calibration-coordinate offset; Outcome gap: calibration minus test outcome mean; Method: estimator; Width (Width SE): mean CI width (its Monte Carlo SE); Coverage (Cov. SE): empirical coverage (its Monte Carlo SE); Labels: mean queried labels.}
\label{tab:miscalibration}
\begin{adjustbox}{max width=\textwidth}
\begin{tabular}{llllllll}
\toprule
Shift & Outcome gap & Method & Width & Width SE & Coverage & Cov. SE & Labels \\
\midrule
0.0 & -0.091 & AM-PPI & 0.1702 & 0.0016 & 0.885 & 0.023 & 405.5 \\
0.0 & -0.091 & AM-PPI (intercept) & 0.1732 & 0.0014 & 0.910 & 0.020 & 415.7 \\
0.25 & +0.256 & AM-PPI & 0.1713 & 0.0018 & 0.895 & 0.022 & 383.6 \\
0.25 & +0.256 & AM-PPI (intercept) & 0.1711 & 0.0015 & 0.880 & 0.023 & 393.7 \\
0.5 & +0.609 & AM-PPI & 0.1719 & 0.0016 & 0.905 & 0.021 & 364.9 \\
0.5 & +0.609 & AM-PPI (intercept) & 0.1723 & 0.0015 & 0.915 & 0.020 & 358.5 \\
0.75 & +0.957 & AM-PPI & 0.1754 & 0.0016 & 0.875 & 0.023 & 387.5 \\
0.75 & +0.957 & AM-PPI (intercept) & 0.1759 & 0.0019 & 0.900 & 0.021 & 364.2 \\
1.0 & +1.298 & AM-PPI & 0.1716 & 0.0017 & 0.895 & 0.022 & 396.6 \\
1.0 & +1.298 & AM-PPI (intercept) & 0.1739 & 0.0019 & 0.905 & 0.021 & 372.7 \\
1.5 & +1.986 & AM-PPI & 0.1758 & 0.00076 & 0.895 & 0.022 & 386.8 \\
1.5 & +1.986 & AM-PPI (intercept) & 0.1729 & 0.0014 & 0.885 & 0.023 & 375.7 \\
2.0 & +2.687 & AM-PPI & 0.2233 & 0.0024 & 0.905 & 0.021 & 319.8 \\
2.0 & +2.687 & AM-PPI (intercept) & 0.2067 & 0.0011 & 0.925 & 0.019 & 367.2 \\
\bottomrule
\end{tabular}
\end{adjustbox}

\end{table}
\floatplacement{table}{tbp}

\clearpage
\section{Synthetic mechanism analysis}
\label{app:ablation}

This appendix decomposes the AM-PPI advantage observed in Section~\ref{sec:experiments} into its underlying mechanisms: cost gap, accuracy gap, candidate subset family, and routing fidelity. We use a hindsight synthetic illustration in which both predictors and the uncertainty model use realized outcomes. This construction exposes a possible routing advantage but does not establish that a deployable router can infer difficulty from covariates alone.

\paragraph{Setup.} We draw $n = 10{,}000$ instances $X \in \mathbb{R}^{5}$ from a standard normal and target $Y$ from the same heteroscedastic generator used for the main synthetic regression experiment (Section~\ref{sec:experiments}). Each predictor is constructed as $f_j(x) = Y + \sigma_j\,\eta_j$ with $\eta_j \sim \mathcal{N}(0, 1)$ drawn i.i.d. across instances and across predictors (we reserve $\eta$ for predictor noise to distinguish it from the data-generating noise $\varepsilon$ in Section~\ref{sec:experiments}). This lets us tune the cheap-versus-expensive accuracy gap directly via $\sigma_1, \sigma_2$ without confounding from predictor fitting. The deployment-time uncertainty model returns the exact realized residual $|Y_i - \lambda_I^\top f_I(X_i)|$ per instance. Using these hindsight quantities, we run the AM-PPI calibration and deployment calculations (Algorithms~\ref{alg:ampi_calib}--\ref{alg:ampi_deploy}) on a $20\%$ test split, sweeping the total budget $B \in [20, 150]$ and averaging $10$ Bernoulli draws per budget. The cheap-predictor cost is fixed at $c_1 = 0.01$. Figure~\ref{fig:ablation_panels} reports CI width versus $B$ for AM-PPI and its two single-predictor restrictions (AM-PPI-1) under six configurations.

\paragraph{Panels (a) and (b): subset structure with identical predictors.} When $\sigma_1 = \sigma_2 = 2.5$ and $c_1 = c_2$ the two predictors are statistically identical. Panel~(a) includes the compound subset $\{f_1, f_2\}$ in $\mathcal{I}$; AM-PPI can exploit realized residual differences between the statistically identical predictors. Panel~(b) restricts $\mathcal{I}$ to singletons; AM-PPI still gains as before this time without the contribution of the compound predictor set. A deployable analogue would require uncertainty estimates informative about these differences.

\paragraph{Panel (c): accuracy gap alone.} Next we increase the accuracy of the 2nd predictor by setting $\sigma_2 = 0.3\,\sigma_1$ at equal cost. This makes the expensive predictor uniformly preferable so that AM-PPI-1(expensive) is below AM-PPI-1(cheap) at every budget and AM-PPI tracks AM-PPI-1(expensive) closely. 

\paragraph{Panel (d): cost gap alone.} Next we increase the cost of the 2nd predictor by setting $c_2 = 4 c_1$ but keep it at equal accuracy. This makes the expensive predictor unaffordable at small budgets ($B \sim50$); a numerical $\hat\pi \geq 0.01$ floor in the implementation produces a flat AM-PPI-1(expensive) plateau in this region rather than a divergence. This is a numerical artifact in that regime. AM-PPI tracks AM-PPI-1(cheap) and the curves converge at large budgets.

\paragraph{Panel (e): cost gap and accuracy gap together.} With differing costs $c_2 = 4 c_1$ and predictor accuracy $\sigma_2 = 0.3\,\sigma_1$ the two effects combine and produce the crossing behavior observed in the synthetic and MIMIC examples. AM-PPI-1(expensive) is infeasible at small $B$, so AM-PPI-1(cheap) wins region $\mathrm{R}_1$. At large $B$ the cost gap is small relative to the budget and AM-PPI-1(expensive)'s lower variance dominates region $\mathrm{R}_3$. In the middle region $\mathrm{R}_2$ AM-PPI outperforms both by routing instances with small cheap-residual to the cheap predictor, freeing budget for labels on harder instances routed to the expensive predictor. This illustrates how a cost--accuracy tradeoff can produce a crossover between singleton strategies.

\paragraph{Panel (f): random routing.} Panel~(f) replaces the hindsight routing $\arg\min$ with a uniform random subset choice for every instance. 

\paragraph{Interpretation.} Cost and accuracy gaps can create a crossover between single-predictor strategies, while informative routing can exploit different strengths across instances. These hindsight curves illustrate that mechanism. 

\begin{figure}[H]
    \centering
    \includegraphics[width=\textwidth]{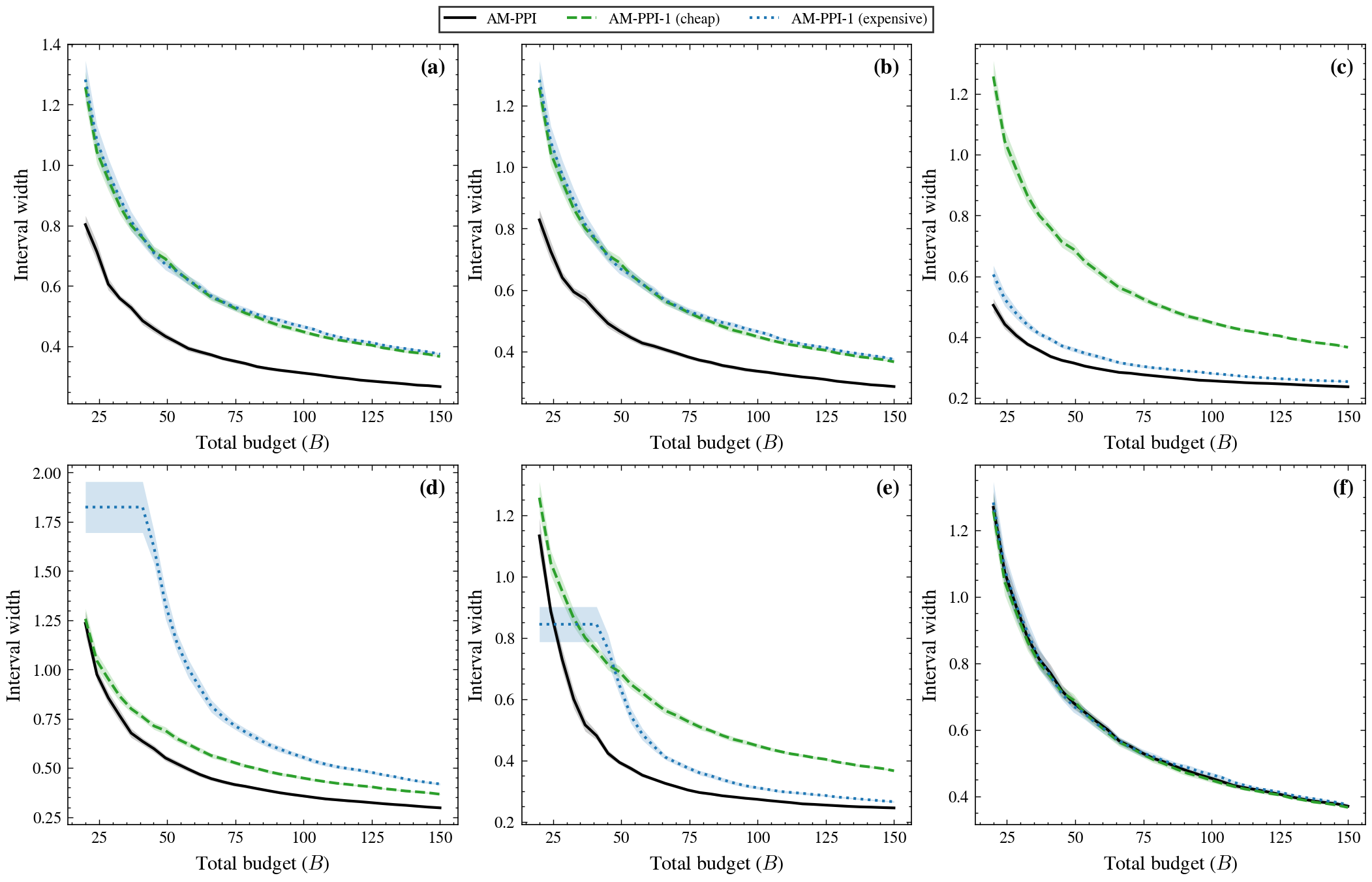}
    \caption{Synthetic mechanism analysis. Each panel sweeps total budget $B$ for one configuration of the cheap and expensive cost ratio, accuracy ratio, candidate subset family $\mathcal{I}$, and routing-error rate; predictors and uncertainty use realized outcomes, making this a hindsight illustration. Panels (a) through (e) build incrementally to a cost-plus-accuracy regime with a crossover between singleton strategies; panel (f) replaces outcome-informed routing with random routing. Shaded bands show $\pm 1$ SEM across $10$ Bernoulli draws per budget.}
    \label{fig:ablation_panels}
\end{figure}

\clearpage
\section{Prompts used in MIMIC-III experiment}
\label{app:mimic_prompts}

We reproduce below the two prompts used in the MIMIC-III EHR--discharge consistency experiment (Section~\ref{sec:experiments}): the few-shot prompt used to generate synthetic ``Labs on Admission'' sections with GPT-OSS-120B and the few-shot prompt used by the Nemotron-3 consistency checkers. Placeholders in braces (e.g., \texttt{\{EHR\}}, \texttt{\{ehr\_data\}}, \texttt{\{clinical\_note\}}) are substituted at inference time. In-prompt few-shot examples are omitted below to avoid reproducing MIMIC-III patient data, per the PhysioNet Credentialed Health Data use agreement.

\subsection{Synthetic ``Labs on Admission'' generation prompt}

\begin{footnotesize}
\begin{verbatim}
You are an expert clinical physician who specializes in writing discharge
summaries at the end of a patient's ICU hospital stay.
These discharge summaries are expansive, semi-structured clinical documents
generated at ICU discharge that detail patient histories, treatments, and
follow-up recommendations.

Your task is to generate the Labs on Admission section of a discharge summary
provided Admission Labs EHR.

CRITICAL CONSTRAINTS
- You must use only the provided structured EHR data, do NOT invent or infer
  information not explicitly supported by the EHR.
- Prefer concise clinical phrasing and typical discharge-summary formatting.
- You should output the lab, lab value, fluid associated with the lab, and
  whether the lab was abnormal or not.
- If a lab is flagged as abnormal, please add a * after the lab.

--------------------------------------------------
OUTPUT REQUIREMENTS
--------------------------------------------------
- Output ONLY the discharge summary section text.
- Match section headers exactly (including colons).
- Follow the formatting style of the example output below.

========================
EXAMPLE
========================
[Few-shot example redacted to comply with the MIMIC-III / PhysioNet
Credentialed Health Data use agreement.]

Now generate for these EHR inputs:

INPUT:
{EHR}

OUTPUT:
\end{verbatim}
\end{footnotesize}

\newpage
\subsection{Consistency checker prompt (few-shot, labs)}

\begin{footnotesize}
\begin{verbatim}
You are an expert Medical Records auditor. Your goal is to verify if the
clinical note is fully supported by the structured Electronic Health Record
(EHR) data.

### Evaluation Framework:
1. **Contradictions (Inconsistent):** Does the EHR say "Diabetes" while the
   note says "No history of Diabetes"?
2. **Numerical Mismatch (Inconsistent):** Do lab values, dosages, or dates in
   the EHR contradict the note?
3. **Omissions (Inconsistent):** If information is in the note but missing
   from the EHR, mark as INCONSISTENT.
4. **Extra EHR Data (Consistent):** The EHR is a comprehensive record that
   can contain information not mentioned in the brief discharge note. Do NOT
   penalize this. As long as everything in the note exists in the EHR, mark
   as 1.
5. **Mutual Absence (Consistent):** If a section in the clinical note is
   empty, blank, or says "None," and the corresponding section in the
   structured EHR is also empty (e.g., [] or no codes listed), this is
   CONSISTENT. Do not mark as an omission if there was nothing to report in
   either source.

Analyze whether all key information documented in the clinical note is
present and fully supported by the structured EHR data.

### Expected Output:
Return ONLY a valid JSON object without markdown formatting, backticks, or
conversational text. It must contain:
- `reasoning`: A brief explanation of your determination, specifically cite
  the conflicting values if `is_consistent` is 0.
- `is_consistent`: 1 (Matches/No contradictions) or 0 (Clear contradiction).
- `certainty`: Your certainty in the is_consistent determination
  (1 = completely uncertain, 2 = somewhat uncertain, 3 = completely certain)

[Few-shot examples redacted to comply with the MIMIC-III / PhysioNet
Credentialed Health Data use agreement.]

Now given the examples above, please provide output for the following input:

### Input Data:
---
#### [STRUCTURED EHR DATA]
{ehr_data}

#### [CLINICAL NOTE]
{clinical_note}

#### [OUTPUT]
\end{verbatim}
\end{footnotesize}

\end{document}